\documentclass[lettersize,journal]{IEEEtran}
\usepackage{amsmath,amsfonts,amssymb,amsthm}
\usepackage{array}
\usepackage{textcomp}
\usepackage{stfloats}
\usepackage{url}\usepackage{hyperref}
\usepackage{verbatim}
\usepackage{graphicx}
\usepackage{cite}
\usepackage{textcomp}
\usepackage{stfloats}
\usepackage{url}
\usepackage{verbatim}
\usepackage{graphicx}
\usepackage{cite}
\usepackage{xcolor}
\usepackage{booktabs}
\newtheorem{lemma}{Lemma}
\newtheorem{theorem}{Theorem}
\newtheorem{corollary}{Corollary}
\usepackage[linesnumbered,ruled,vlined]{algorithm2e}
\SetKwInput{KwInput}{Input}                % Set the Input
\SetKwInput{KwOutput}{Output}  
\usepackage{algpseudocode}
\usepackage{url}
\usepackage{multirow}
\usepackage{enumerate}
\usepackage{diagbox}
\usepackage{colortbl}
\usepackage{subcaption}
\begin{document}

\title{Uncertainty Calibration for Counterfactual Propensity Estimation in Recommendation}

%\author{IEEE Publication Technology,~\IEEEmembership{Staff,~IEEE,}
        % <-this % stops a space
%\thanks{This paper was produced by the IEEE Publication Technology Group. They are in Piscataway, NJ.}% <-this % stops a space
%\thanks{Manuscript received April 19, 2021; revised August 16, 2021.}}
\author{Wenbo Hu,~\IEEEmembership{Member,~IEEE,} Xin Sun,~\IEEEmembership{Member,~IEEE,} Qiang Liu,~\IEEEmembership{Member,~IEEE,}\\ Le Wu,~\IEEEmembership{Member,~IEEE,} Liang Wang,~\IEEEmembership{Fellow,~IEEE,}
\thanks{This work was supported by the National Science and Technology Major Project (No. 2021ZD0111802), National Natural Science Foundation of China (No. 62306098),  the Open Projects Program of State Key Laboratory of Multimodal Artificial Intelligence Systems, Funds for the Central Universities (No. JZ2024HGTB0256),
the Open Project of Anhui Provincial Key Laboratory of Multimodal Cognitive Computation, Anhui University (No. MMC202412).}
\thanks{Wenbo Hu and Le Wu are with the School of Computer and Information, Hefei University of Technology, Hefei, China (Email: \{wenbohu,lewu\}@hfut.edu.cn). Xin Sun is with the University of Science and Technology of China (Email: sunxin000@mail.ustc.edu.cn) . Qiang Liu and Liang Wang are with the Institute of Automation, Chinese Academy of Sciences(Email: \{qiang.liu,wangliang\}@nlpr.ia.ac.cn).}
\thanks{The first two authors contributed equally. Corresponding author: Qiang Liu.}
\thanks{{The code and data for reproducing the results presented in this paper are publicly available at: \url{https://github.com/sunxin000/uncertainty4recsys}.}}
}

\markboth{Journal of \LaTeX\ Class Files,~Vol.~14, No.~8, August~2021}%
{Shell \MakeLowercase{\textit{et al.}}: A Sample Article Using IEEEtran.cls for IEEE Journals}


\maketitle

\begin{abstract}
Post-click conversion rate (CVR) is a reliable indicator of online customers' preferences, making it crucial for developing recommender systems. A major challenge in predicting CVR is severe selection bias, arising from users' inherent self-selection behavior and the system's item selection process. To mitigate this issue, the inverse propensity score (IPS) is employed to weight the prediction error of each observed instance. However, current propensity score estimations are unreliable due to the lack of a quality measure. To address this, we evaluate the quality of propensity scores from the perspective of uncertainty calibration, proposing the use of Expected Calibration Error (ECE) as a measure of propensity-score quality, which quantifies the extent to which predicted probabilities are overconfident by assessing the difference between predicted probabilities and actual observed frequencies. Miscalibrated propensity scores can lead to distorted IPS weights, thereby compromising the debiasing process in CVR prediction.
In this paper, we introduce a model-agnostic calibration framework for propensity-based debiasing of CVR predictions. Theoretical analysis on bias and generalization bounds demonstrates the superiority of calibrated propensity estimates over uncalibrated ones. Experiments conducted on the Coat, Yahoo and KuaiRand datasets show improved uncertainty calibration, as evidenced by lower ECE values, leading to enhanced CVR prediction outcomes.
\end{abstract}
%In recommendation systems, addressing selection bias is crucial for improving fairness and diversity. Propensity score-based debiasing methods, such as the Inverse Propensity Scoring (IPS) approach, often suffer from high variance and miscalibration issues. This paper introduces an uncertainty calibration framework for propensity score estimation, utilizing Expected Calibration Error (ECE) to assess the quality of propensity scores. We explore three calibration methods: MC Dropout, Deep Ensembles, and Platt Scaling, comparing their effectiveness in reducing calibration error and improving recommendation accuracy. Theoretical analysis demonstrates that calibrated IPS exhibits smaller bias, and experimental results on two representative datasets, Coat Shopping and Yahoo! R3, validate the significant improvement in recommendation performance with calibrated propensity scores. Additionally, we show that our method can enhance other state-of-the-art CVR prediction models, further proving the robustness and generalizability of our approach. Efficiency experiments indicate that Platt Scaling and MC Dropout are computationally efficient, while Deep Ensembles, though effective, incur higher computational costs. Our findings highlight the importance of accurate propensity score calibration in mitigating selection bias and enhancing recommendation system performance.


\section{Introduction}
The post-click conversion rate (CVR) represents the likelihood of a user consuming an online item after clicking on it. Predicting CVR is essentially a counterfactual problem, as it involves estimating the conversion rates of all user-item pairs under the hypothetical scenario that all items are clicked by all users. However, this scenario contradicts reality due to selection bias. Users freely choose which items to rate, resulting in observed user-item feedback that is not representative of all possible user-item pairs. Consequently, the feedback data is often missing not at random (MNAR)~\cite{guo2021enhanced, dai2022generalized, wang2022escm, zhou2023generalized, liu2022rmt}.

\begin{figure}[htbp]
\centering
\subcaptionbox{Coat Shopping}{\includegraphics[width=.45\columnwidth]{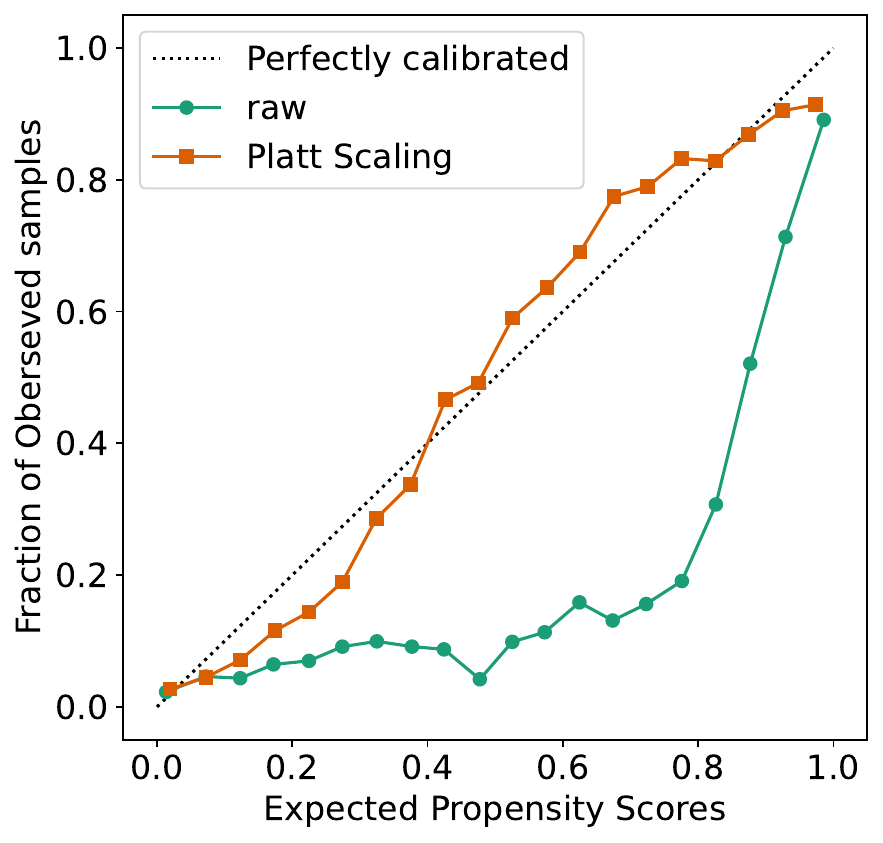}}
\subcaptionbox{{Yahoo R3!}}{\includegraphics[width=.45\columnwidth]{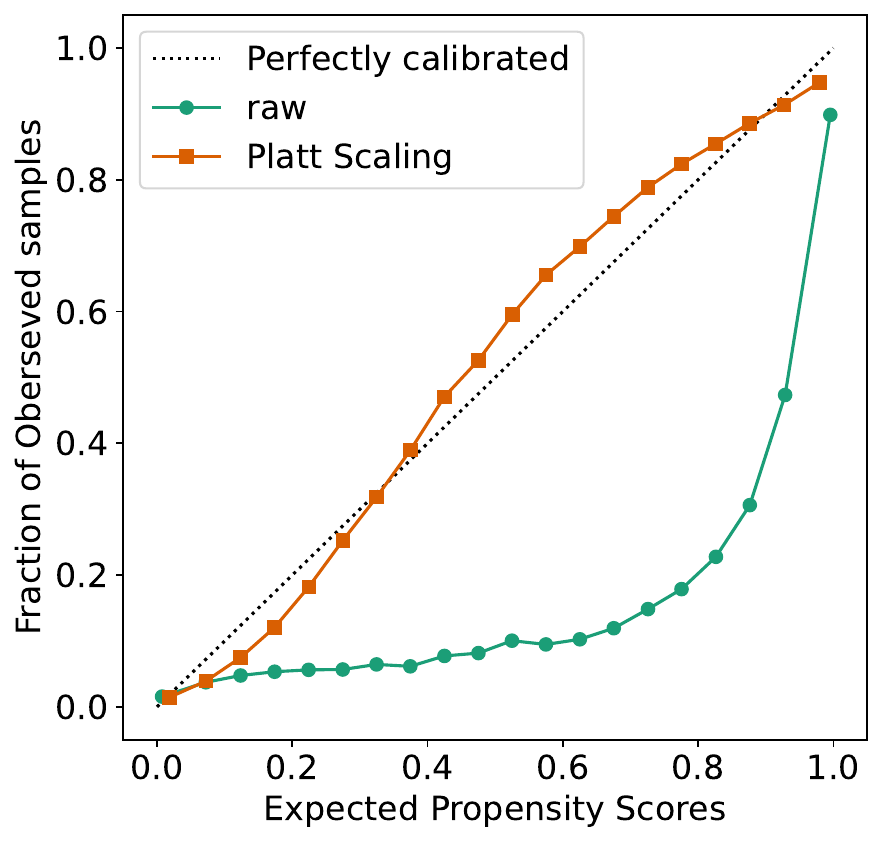}}
\caption{For recommendation with MNAR on the Coat shopping dataset, we use the raw  propensity estimator with and without the platt scaling calibration and give the scatter plot of the expected propensity vs the fraction of observed ratings. The diagonal line is the perfect uncertainty calibration result. As can be seen, the raw propensity estimations are severely miscalibrated.}
\label{fig:intro-miscalibration}
\end{figure}

To address this problem, the inverse propensity score (IPS) approach is employed to handle selection bias~\cite{seaman2013review, little2019statistical}. This approach treats recommendation as an intervention, analogous to treating a patient with a specific drug. In both scenarios, we have only partial knowledge of how certain treatments (items) benefit certain patients (users), with outcomes for most patient-treatment (user-item) pairs remaining unobserved. For recommendations, IPS inversely scores the prediction error of each feedback using the propensity of that feedback~\cite{swaminathan2015self, schnabel2016recommendations}. Doubly robust (DR) learning approaches, which combine IPS and error imputation (EIB) methods, also achieve state-of-the-art performance in debiasing CVR prediction~\cite{wang2019doubly, saito2020doubly, guo2021enhanced, dai2022generalized}. The robustness and accuracy of propensity estimates are crucial for propensity-based debiasing in recommendation systems. Unfortunately, there is no systematic investigation into reliable quality measures for propensity scores. As a result, miscalibrated propensity score estimates are often overlooked, potentially diminishing the effectiveness of debiasing methods.

In machine learning methods widely used in recommendation systems, uncertainty quantification is often poorly characterized, leading to over-confident predictions. This issue is prevalent not only in deep learning models~\cite{guo2017calibration} but also in shallow models, such as logistic regression~\cite{vaicenavicius2019evaluating}. {Both types of models are prone to overconfidence because they are typically optimized to minimize error metrics without explicitly accounting for uncertainty, resulting in predictions that underestimate the true uncertainty in the data~\cite{guo2017calibration}.} The uncertainty of personalized ranking probabilities can be learned through uncertainty calibration methods~\cite{menon2012predicting, kweon2022obtaining} and applied in online advertising systems~\cite{wei2022posterior, xu2022ukd}.

{Calibration in machine learning refers to the degree to which predicted probabilities reflect the true likelihood of an event. A model is considered well-calibrated if, for predictions assigned a probability of $p$, the actual frequency of the event is also $p$.}
The propensity scores are frequently miscalibrated, limiting the effectiveness of IPS, even though IPS has been validated in recommendation systems and other applications. 
As illustrated in Fig.~\ref{fig:intro-miscalibration}, expected propensity scores are not calibrated with the fraction of observed samples, deviating from perfect calibration (the diagonal line). {Please note that the results in Fig.~\ref{fig:intro-miscalibration} also applies to the more recent methods, DR-JL~\cite{wang2019doubly} and MRDR~\cite{guo2021enhanced}}. In terms of uncertainty calibration, expected propensity scores, such as 95\%, should correspond to the same level of observed sample fraction (95\%). The uncertainty originates from inaccurate propensity score predictions, leading to inaccurate recommendations when dealing with MNAR data using miscalibrated propensity scores.

%In this paper, we propose using expected calibration error (ECE) as a quality measure for the reliability of propensity scores and highlight the miscalibration issue in current propensity estimation. Additionally, we introduce a model-agnostic uncertainty calibration framework for propensity estimation. Extensive experiments demonstrate that lower ECE values in propensity scores result in better CVR prediction outcomes.

In machine learning methods commonly used in recommendation systems, both deep and shallow models are prone to overconfidence, often underestimating the true uncertainty in the data. This overconfidence stems from optimizing error metrics without explicitly accounting for uncertainty, which can lead to miscalibrated predictions~\cite{guo2017calibration, vaicenavicius2019evaluating}. Propensity scores, crucial for IPS-based debiasing, are particularly susceptible to this issue, as demonstrated in Fig.~\ref{fig:intro-miscalibration}, where the observed miscalibration deviates significantly from perfect uncertainty calibration. Miscalibrated propensity scores result in distorted IPS weights, ultimately compromising the debiasing process and prediction reliability. To address this, we propose using Expected Calibration Error (ECE) as a robust measure of propensity-score quality. By quantifying the degree of miscalibration, ECE reveals the extent to which overconfidence in propensity predictions hampers IPS effectiveness. Lower ECE values indicate better-calibrated scores, enabling more accurate and unbiased CVR predictions. This highlights the critical importance of addressing miscalibration to enhance the reliability and robustness of IPS-based methods in handling MNAR data.

%{The proposed method operates within a counterfactual framework by addressing the critical challenge of propensity score estimation, a key component of Inverse Propensity Scoring (IPS). IPS reweights observed data to approximate outcomes under unobserved, counterfactual scenarios, enabling unbiased evaluation in recommendation tasks. Our uncertainty-calibrated propensity score estimator improves the reliability of this reweighting process, reducing bias and ensuring robust counterfactual inferences. By enhancing the accuracy of propensity scores, our method strengthens the foundation for solving counterfactual problems in recommendation systems.}

The contributions of this paper are as follows:
\begin{itemize}
  \item  {\textbf{Identification of Propensity Miscalibration}: We reveal a critical issue in current CVR prediction approaches—the miscalibration of propensity scores—and propose \textbf{Expected Calibration Error (ECE)} as a robust metric to assess the reliability of these scores.}
  \item  {\textbf{Uncertainty Calibration Framework}: We introduce a novel, \textbf{model-agnostic} framework for uncertainty calibration of propensity scores, significantly improving their reliability and aligning them more closely with observed data distributions.}
  \item  {\textbf{Enhanced Debiasing Effectiveness}: By addressing miscalibration, we demonstrate how calibrated propensity scores bolster the debiasing performance of inverse propensity score (IPS)-based and doubly robust (DR) learning methods in recommendation systems.}
  \item  {\textbf{Comprehensive Theoretical and Empirical Validation}: We provide rigorous theoretical analysis of bias reduction and generalization bounds, supported by extensive experiments on benchmark datasets (e.g., Coat, Yahoo, and KuaiRand). Results highlight the superiority of calibrated propensity scores in achieving accurate and unbiased CVR predictions, as evidenced by lower ECE values and improved prediction performance.}
\end{itemize}
\iffalse..\begin{itemize}
    \item We identify the miscalibration issue in propensity estimation and propose using ECE as a quality measure for the reliability of propensity scores.
    \item We demonstrate that miscalibration of propensity scores limits the debiasing performance of propensity-based CVR prediction and propose a thoughtful uncertainty calibration methodology for propensity scores.
    \item We provide theoretical analysis and experimental results to show the superiority of uncertainty-calibrated propensity scores for unbiased CVR prediction.
\end{itemize}
\fi
\section{Preliminaries}
In this section, we introduce the preliminaries of counterfactual propensity estimation and uncertainty quantification. 
Table I in Supplemental Materials describes the main symbols used in this paper.

\subsection{Propensity-based Debiasing Recommendation}
Let $\mathcal{U} = \{u_1, u_2,\dots, u_m\}$ and $\mathcal{I} = \{i_1, i_2, \dots, i_n\}$ be the sets of $m$ users and $n$ items. The set of user-item pairs is denoted as $\mathcal{D} = \mathcal{U} \times \mathcal{I}$. We use $\mathbf{R} \in \{0,1\}^{m\times n}$ to represent the conversion matrix where each entry $r_{u,i}$ indicates an observed conversion label. Let $\hat{\mathbf{R}} \in [0,1]^{m\times n}$ be the predicted conversion rate matrix and each entry $\hat{r}_{u,i} \in [0,1]$ represent the predicted conversion rate, which is obtained by the conversion model $f_\theta$ with parameter $\theta$. Additionally, we denote $o_{u,i}$ as the click event indicator and $\mathcal{O}$ as the click label matrix. We denote the  observed conversion label matrix as $\mathbf{R^{o}} = \mathbf{R} \odot \mathcal{O}$, where $\odot$ is Hadamard product operator. If all conversion labels are available, the prediction errors $\mathbf{E} = \{e_{u,i} | (u, i) \in \mathcal{D}\}$ can be calculated, the ideal loss function is:
\begin{equation}
\label{eqn:ideal}
\mathcal{L}_{ideal}(\mathbf{\hat{R}}, \mathbf{R}) = \frac{1}{|\mathcal{D}|} \sum_{u,i\in \mathcal{D}} e_{u,i},   
\end{equation}
where $e_{u,i}$ is the prediction error and we adopt the cross entropy in this paper. 
We adopt the cross entropy $e_{u,i} = CE(r_{u,i}, \hat{r}_{u,i}) = -r_{u,i}\log \hat{r}_{u,i} - (1-r_{u,i}) \log (1-\hat{r}_{u,i})$. 

In practice, only part of the conversion label are available. The naive estimate of ideal loss function is to  averages the prediction errors of the available items:
\begin{equation}
  \label{eqn:naive}
    \mathcal{L}_{naive}(\mathbf{\hat{R}}, \mathbf{R}) =\frac{1}{|\mathcal{D}|}\sum_{o_{u,i}=1,u,i\in\mathcal{D}}e_{u,i} = \frac{1}{|\mathcal{D}|} \sum_{(u,i)\in \mathcal{D}} o_{u,i}e_{u,i}.
\end{equation}
The naive estimator is biased when the conversion labels are Missing Not At Random which is resulted from the selection biases of the real recommendation system~\cite{marlin2007collaborative}, i.e., 
\begin{equation}
    \mathbb{E}_\mathcal{O} [\mathcal{L}_{naive}(\mathbf{\hat{R}}, \mathbf{R})] \neq \mathcal{L}_{ideal}(\mathbf{\hat{R}}, \mathbf{R}).
\end{equation}

To reduce the selection bias of the naive estimator, the inverse propensity score considers reweighting the error of the observed ratings of the inverse propensity score~\cite{swaminathan2015self,sato2020unbiased}. In CVR prediction task, $p_{u,i}$ represents the probability of a user $u$ clicks an item $i$ and $p_{u,i} = \mathbb{P} (o_{u,i}=1) = \mathbb{E}[o_{u,i}]$, which is also known as click-through rate (CTR) in the CVR prediction task setting. 
Specifically, the $p_{u,i}$  is estimated using a machine learning classifier $g_\phi$, such as naive Bayes. We call the model $g_\phi$ as propensity estimation model. The estimated value of $p_{u,i}$ is denoted as $\hat{p}_{u,i}$. The matrices $\mathcal{P}$ and $\mathcal{\hat{P}}$ represent the propensity score matrix and estimated propensity score matrix, respectively.
With the inverse propensity scores, the prediction error of IPS is obtained via:
\begin{equation}
  \label{eqn:IPS}
    \mathcal{L}_{\textrm{IPS}}(\mathbf{\hat{R}}, \mathbf{R})=\frac{1}{|\mathcal{D}|}\sum_{(u,i)\in\mathcal{D}}\frac{o_{u,i}e_{u,i}}{\hat{p}_{u,i}}.
\end{equation}

% Another line of research aim to use an imputation model to predict the more accurate rating predictions~\cite{steck2010training}. The imputation error is denoted as $\hat{e}_{u,i}$ and this the  error-imputation-based (EIB) estimator has the following prediction error: 
% \begin{equation}
%     \mathcal{L}_{\text{EIB}}(\mathbf{\hat{R}}, \mathbf{R})=\frac{1}{|\mathcal{D}|}\sum_{u,i\in\mathcal{D}}\left(o_{u,i}e_{u,i}+(1-o_{u,i}\hat{e}_{u,i})\right).
% \end{equation}
A more recent progress, the doubly robust estimator, is to combine the IPS and the error-imputation-based (EIB) estimators via joint learning to have the best of the both worlds~\cite{wang2019doubly, guo2021enhanced}. Given the imputed errors $\mathbf{\hat{E}} = \{\hat{e}_{u,i} | (u, i) \in \mathcal{D}\}$ , its loss function is formulated as:
\begin{equation}
    \mathcal{L}_{\text{DR}}(\mathbf{\hat{R}}, \mathbf{R})=\frac{1}{|\mathcal{D}|}\sum_{u,i\in\mathcal{D}}\left(\hat{e}_{u,i}+\frac{o_{u,i}(e_{u,i}-\hat{e}_{u,i})}{\hat{p}_{u,i}}\right).
    \label{eq: dr}
\end{equation}
Doubly robust joint learning(DR-JL) approach\cite{wang2019doubly} estimates the CVR prediction model $f_\theta$ and error imputation model $\hat{e}_{u,i} = h_\psi(x_{u,i}) $ alternately: given $\hat{\psi}$ , $\theta$ is updated by minimizing Eqn.~\ref{eq: dr}; given $\hat{\theta}$, $\psi$ is updated by minimizing:
\begin{equation}
\mathcal{L}_e^{D R-J L}(\psi)=\sum_{u, i\in \mathcal{D}} \frac{o_{u, i}\left(\hat{e}_{u, i}-e_{u, i}\right)^2}{\hat{p}_{u, i}}
\label{eq: drjl}
\end{equation}

Recently, the more robust doubly robust (MRDR) method~\cite{guo2021enhanced} enhances the robustness of DR-JL by optimizing the variance of the DR estimator with the imputation model. Specifically, MRDR keeps the loss of the CVR prediction model in Eqn.~\ref{eq: dr} unchanged, which replaces the loss of the imputation model in Eqn.~\ref{eq: drjl} with the following loss
\begin{equation}
\mathcal{L}_e^{M R D R}(\theta)=\sum_{(u, i) \in \mathcal{D}} \frac{o_{u, i}\left(\hat{e}_{u, i}-e_{u, i}\right)^2}{\hat{p}_{u, i}} \cdot \frac{1-\hat{p}_{u, i}}{\hat{p}_{u, i}}
\label{eq: mrdr}
\end{equation}
This substitution can help reduce the variance of Eqn.~\ref{eq: dr} and hence get a more robust estimator.

% The inverse propensity score is applied to the recommendation, post-click conversion rate and click-through rate with MNAR~\cite{yuan2019improving,guo2021enhanced}.
\subsection{Trustworthy Machine Learning and Probability  Uncertainty for Relibility}
Machine learning, particularly deep learning methods, has achieved pervasive success in various domains, including vision, speech, natural language processing, control, and computer Go~\cite{lecun2015deep,silver2016mastering}. Despite their dominant prediction performance across these areas~\cite{lecun2015deep,liu2017multi}, such as computer vision, natural language processing, and recommendation systems, deep learning models often produce overconfident and miscalibrated predictions~\cite{guo2017calibration}. Overconfident predictions can undermine the accuracy, robustness, and reliability of these models.
Therefore, it is imperative to characterize the uncertainty in deep learning models~\cite{kendall2017uncertainties,abdar2021review}. Safety-critical tasks are ubiquitous, including autonomous driving~\cite{leibig2017leveraging}, medical diagnoses~\cite{michelmore2020uncertainty}, weather forecasting~\cite{gneiting2005weather}, load forecasting~\cite{taylor2002neural}, social network analysis~\cite{akcora2019graphboot}, anomaly detection~\cite{xu2024calibrated}, and traffic flow forecasting~\cite{qian2023towards}. In these real-world application scenarios, diverse probabilistic uncertainties in model predictions arise from measurement noise, external changes, data missingness, etc. This necessitates that deep learning models not only produce accurate predictions but also provide insights into the reliability of these predictions in terms of uncertainty.

In machine learning, there are two types of uncertainty: \emph{aleatoric} uncertainty and \emph{epistemic} uncertainty (also known as data uncertainty and model uncertainty)~\cite{kendall2017uncertainties}. Aleatoric uncertainty captures the inherent noise in the data, which may arise from sources such as sensor noise or motion noise. Epistemic uncertainty, on the other hand, pertains to the uncertainty in the model parameters and structure. It can be fully captured given sufficient data. In many scenarios, epistemic uncertainty is commonly referred to as model uncertainty.

\iffalse
Specifically, a supervised machine learning model learns a mapping function $f$ from a feature vector $x$ to a label $y$: $  \hat{y}=f(x)$.
The label $y$ can be either discrete or continuous, which corresponds to classification or regression problems respectively. 
A widely-used mapping function is a linear feature combination plus an activation function $\sigma$:
$\hat{y}=f(x)=\sigma\left(\textbf{w}^{\top}x\right)$,
where $\textbf{w}$ is a coefficient vector. 
A representative example is a binary classification method, logistic regression whose activation function is the sigmoid function.
Based on the simple structure, we obtain multifarious deep learning methods, by stacking multiple layers of the coefficient vectors and using some advanced units, such as CNN, RNN and attention~\cite{lecun2015deep}.

In ordinary deep learning settings, only the point estimates are produced by $f$. However, for many safety-critical tasks, the uncertainty estimates are of great significance. There are two types of uncertainty quantification for deep learning: classification and regression:
\begin{itemize}
	\item UQ for classification: for a k-class classification task, UQ should estimate the assigning Dirichlet probabilities for each class: $$\textrm{Dir}(\alpha_1, \alpha_2, \cdots, \alpha_k), \sum_{i=1}^{k}\alpha_i=1;$$
	\item UQ for regression: for a real-value regression problem, UQ should estimate a specific probability distribution (for example Gaussian), or a series of quantiles.  
\end{itemize}
\fi

\subsection{Uncertainty Calibration for Deep Learning}
To formalize, the propensity score is well-calibrated if it equals the correctness ratio of the available conversion labels~\cite{kull2017beta}. For instance, if the propensity estimation model $g_\phi$ outputs 100 predictions, each with a confidence (i.e., uncalibrated propensity score) of 0.95, then 95\% of the conversion labels are expected to be available. We define perfect calibration of propensity estimation as:

\begin{equation}
    \mathbb{P}(o = 1 | \hat{p} = p) = p, \quad \forall p \in [0,1],
\end{equation}
where $\hat{p}$ is the output of $g_\phi$. Miscalibration can be measured by the Expected Calibration Error (ECE), which is the expectation of the coverage probability difference of the prediction intervals. In practice, we partition propensity predictions into $M$ bins of equal width and calculate the weighted sum of all bins via:
\begin{equation}
\label{eqn:ece}
\mathrm{ECE(g_\phi)}=\sum_{m=1}^M \frac{\left|B_m\right|}{n}\left|\operatorname{freq}\left(B_m\right)-\operatorname{conf}\left(B_m\right)\right|,
\end{equation}
where $n$ is the number of samples and $B_m$ is the set of indices of samples whose propensity prediction falls into the interval $I_m = (\frac{m-1}{M}, \frac{m}{M}]$. $\operatorname{conf}(B_m)$ and $\operatorname{freq}(B_m)$ are defined as:
\begin{gather}
    \operatorname{conf}\left(B_m\right) = \frac{1}{|B_m|} \sum_{u, i \in B_m} \hat{p}_{u,i} \\
    \operatorname{freq}\left(B_m\right) = \frac{1}{|B_m|} \sum_{u, i \in B_m} \mathbf{1}(o_{u,i} = 1),
\end{gather}

% TODO, ece as the indicator.....

Regarding calibration methodologies, one effective approach is Bayesian generative modeling, with representative models including Bayesian neural networks and deep Gaussian processes~\cite{wilson2020bayesian,wang2020survey}. Bayesian neural networks are generally computationally expensive to train, so approximate methods have been developed, such as MC-Dropout~\cite{kendall2017uncertainties} and deep ensembles~\cite{lakshminarayanan2017simple}. Alternatively, uncertainties can be obtained from the calibration of inaccurate uncertainties. Methods employing scaling and binning for calibration are used for both classification and regression models~\cite{platt1999probabilistic,niculescu2005predicting,guo2017calibration,kuleshov2018accurate,cui2020calibrated,he2023investigating,liu2023deep,li2024conformalized,cui2024sde,chen2024unveiling}. An additional advantage of calibration methods is their model-agnostic nature, making them applicable to any IPS-based model.
{However, the joint modeling of the model calibration and the base model, namely the CVR model in this paper, might bring synergy between the two tasks. The possible way includes the fully Bayesian generative modeling of the two parts or transfrom the uncertainty calibration as a loss regularization term in the CVR modeling, where both faces several challenges for modeling and training.}

\section{Uncertainty Calibration for Propensity Estimation}
In this section, we present our approach to counterfactual propensity estimation with uncertainty calibration. We also provide a theoretical guarantee for our uncertainty calibration model in recommendation systems with Missing Not At Random (MNAR) data.

\subsection{Propensity Estimation Procedure}
The propensity probability $p$ is critical for the inverse propensity score, which ensures the unbiasedness of the IPS estimator when the inverse propensity score is accurate~\cite{vermeulen2015bias}. Propensities are learned using a machine learning model $g_\phi: x_{u,i} \rightarrow p_{u,i}$, where $p_{u,i} \in [0, 1]$. This model can be naive Bayes, logistic regression, or deep neural networks. We utilize neural networks to fit $g_\phi$. The objective is to find model parameters $\phi$ that maximize $P(\mathcal{O}|X, \phi)$, where $x_{u,i}$ is the vector encoding all observable information about a user-item pair and $X$ is the set of such vectors. The loss function is given by:
\begin{equation}
\label{eqn:ctr}
    \mathcal{L}_{g_\phi} = -\sum_{(u,i) \in \mathcal{D}} [o_{u,i} \cdot \log \hat{p}_{u,i} + (1 - o_{u,i})\cdot \log (1-\hat{p}_{u,i})] 
\end{equation}
\subsection{Uncertainty Calibration for Propensity Estimation}
As illustrated in Fig.~\ref{fig:intro-miscalibration}, raw propensity estimation models are generally miscalibrated, often producing overconfident probability predictions.
 { The reason of the overconfident and miscalibrated propensity estimation models is that they are typically optimized to minimize error metrics without explicitly accounting for uncertainty, which shares the same reason of the overconfidence of the conventional deep models.}
To reduce biases and achieve calibrated propensity scores, we consider a model-agnostic uncertainty calibration $q$ in conjunction with the propensity learning model $g_\phi$. Specifically, two model-agnostic methods are considered for propensity probability calibration: 1) uncertainty probability quantification and 2) post-processing uncertainty calibration.

\subsubsection{Uncertainty Probability Quantification for Propensity scores}
The uncertainty probability quantification considers a generative probability quantification model $q(P|\Theta)$, where $\Theta$ represents the model parameters.

{Due to the challenges of performing exacting inference and its high computational cost associated with fully Bayesian models}, we propose two approximate uncertainty quantification methods for propensity estimation: 1) Monte Carlo Dropout~\cite{gal2016dropout} , 2) deep ensembles~\cite{lakshminarayanan2017simple} and 3) dual focal loss~\cite{tao2023dual}.

\textbf{Monte Carlo (MC) Dropout} involves randomly deactivating neurons during testing in the originally trained deep neural network. Multiple samples ($T$) are taken to produce an approximate posterior distribution through model averaging:
\begin{equation}\label{eqn:model_average}
  q(p|x,\Theta)\sim\frac{1}{T}\sum_{t}^{T} q_{t}(p|x, g_\phi(x), \Theta_{t}).
\end{equation}

\textbf{Deep Ensembles} involve training multiple model replicas with different random initializations, without interactions during training. The approximate propensity probability distribution is obtained by combining and averaging the replicas as shown in Eqn.~\ref{eqn:model_average}. Compared to MC-Dropout, deep ensembles tend to perform better because the model ensembles learn distinct model distributions, whereas MC-Dropout only varies during the testing stage. However, deep ensembles are generally more computationally expensive since the models are trained multiple times.

 {\textbf{Dual Focal Loss}\cite{tao2023dual} not only considers the ground truth logit, but also take into account the highest logit ranked after the ground truth logit. By maximizing the gap between these two logits, our dual focal loss can achieve a better balance between over-confidence and under-confidence. } 

\subsubsection{Post-processing Calibration for Propensity scores}
In addition to direct uncertainty quantification, post-processing calibration can be applied to derive accurate predictive uncertainties from inaccurate softmax probabilities (or other model output probabilities)~\cite{platt1999probabilistic, guo2017calibration}.

\textbf{Platt Scaling} adjusts the original propensity outputs to learn accurate inverse propensities via:
\begin{equation}
    \label{eqn:platt}
    q(g_\phi(x))=\sigma(b\cdot g_\phi(x)+c),
\end{equation}
where $g_\phi(x)$ represents the original propensity outputs, $\sigma$ is the sigmoid function, and $b, c$ are learnable parameters of the sigmoid function~\cite{platt1999probabilistic}. { With the goal of achieving better alignment between predicted and true probabilities, parameters $b$ and $c$ are optimized using a negative log-likelihood (NLL) objective on a held-out calibration dataset. This process enables the model to learn a calibration mapping that minimizes overconfidence or underconfidence in predictions, thereby enhancing overall calibration performance.} Platt scaling is equivalent to class-conditional Gaussian likelihoods with the same variance. For multi-class classification, Platt scaling can be augmented with a temperature parameter to soften the softmax output, known as temperature scaling~\cite{guo2017calibration}. In this paper, we employ Platt scaling for the NMAR binary setting.
\begin{algorithm}
\caption{Uncertainty calibration for IPS in CVR prediction task}\label{alg:cap}

\KwInput{$X$: set of item-user features, $\mathcal{O}$: click label matrix, $\mathbf{R^o}$: observed conversion label matrix}
\KwOutput{$\theta$}
Initialize the parameter $\phi,\theta$\;
\For{number of steps for training propensity estimation model $g_\phi$}
{
Sample a batch from $X$ and $\mathcal{O}$\;
Update $\phi$ by descending along the gradient $\nabla_\phi \mathcal{L}_{g_\phi}(\phi)$\;
}

\If{Uncertainty Quantification is used}
{
    Obtain multiple model ensemble/non-ensemble model using Eqn.~\ref{eqn:model_average}\;

}
\ElseIf{Post-processing Calibration is used}
{
    Calibrating the overconfident predicts to calibrated ones using Eqn.~\ref{eqn:platt}\;
}
Output propensity scores $\mathcal{\hat{P}}$ using $g_\phi$ for observed samples\;
\For{number of steps training the CVR prediction model $f_\theta$}
{
Sample a batch from $\mathbf{R^o}$ and $\mathcal{\hat{P}}$\;
Update $\theta$ by descending along the gradient $\nabla_\theta \mathcal{L}_{\textrm{IPS}}(\theta)$\;
}

\end{algorithm}

With the calibrated propensity scores, we can train the propensity-based CVR prediction debiasing model in two steps: first, train the propensity estimation model $g_\phi$, obtain the calibrated propensity scores, and then train the CVR prediction model $f_\theta$ using these inverse calibrated propensities. This process is detailed in Algorithm~\ref{alg:cap}. {Please note that, if the calibration method is Dual Focal loss, the loss gradient in line 4 would consider the added dual focal loss and no other calibration step (line 5-8) is needed. }

{Our proposed method is built upon but differs from existing calibration approaches in several key ways. First, it is model-agnostic, enabling application across various propensity-based models without altering their underlying architectures. Second, unlike traditional calibration methods focused on general classification tasks, our approach specifically addresses propensity score calibration for post-click conversion rate (CVR) prediction, tackling the selection bias inherent in recommendation systems. Third, it uniquely integrates with debiasing techniques such as Inverse Propensity Scoring (IPS) and Doubly Robust (DR) learning, enhancing their effectiveness by resolving miscalibrated propensity scores.}

{Our method is built upon the exsting calibration approachs but stands out by being model-agnostic, specifically targeting propensity score calibration for CVR prediction to address selection bias in recommendation systems. Unlike existing methods, it integrates with debiasing techniques like IPS and DR learning, improving their effectiveness. Additionally, it prioritizes Expected Calibration Error (ECE) as a key metric, offering a focused evaluation of calibration quality in the context of the counterfactual propensity estimation in recommendation.}

\subsubsection{Computational Complexity Analysis}
\label{sec:computational complexity}
The computational complexity of the calibration methods for propensity score estimation varies significantly across techniques. The choice of calibration method greatly impacts computational costs. Deep ensembles are likely the most computationally expensive due to the necessity of multiple training cycles, followed by Monte Carlo Dropout, which scales with the number of samples. Post-processing calibration methods typically involve lighter computations on the outputs of an existing model. Below is a comprehensive analysis of each method.
\begin{itemize}
    \item {Monte Carlo Dropout involves sampling the model output multiple times (denoted as $T$) with randomly deactivated neurons. Each sample incurs a forward pass through the neural network, thus making the computational cost proportional to $T$ times the cost of a single forward pass. The complexity is therefore $O(T \times C)$ where $C$ represents the computational cost of one forward pass through the network.}
    \item {Deep Ensembles, on the other hand, requires training multiple independent models from scratch with different initializations. Assuming each model has a training complexity of $O(M)$, where $M$ represents the training complexity of one model (typically including several epochs and forward-backward passes), and there are $N$ such models, the total computational cost would be $O(N \times M)$. The cost can be substantially higher than Monte Carlo Dropout, especially if $N$ and the complexity of individual model training are large. To mitigate the high computational demand of traditional deep ensembles, the BatchEnsemble method can be incorporated, which shares parameters across different models in the ensemble, thereby reducing both memory usage and computational overhead while preserving model diversity~\cite{wen2019batchensemble}.}
    \item {Post-Processing Calibration (e.g., Platt Scaling) involves adjusting the outputs of an already trained model using additional parameters (like $b$ and $c$ in Platt scaling). The primary computational expense here is the forward pass to compute $g_\phi(x)$ and the subsequent optimization to learn the calibration parameters. This can generally be much less computationally intensive compared to the previous methods, as it typically involves simpler operations over the model’s outputs and potentially fewer parameters to optimize.}
    {Dual Focal Loss optimizes both the ground truth logit and the highest competing logit by maximizing the gap between them. Since it modifies the loss function during training, the computational complexity is comparable to standard training with additional gradient computations to handle the second logit. It remains computationally less expensive than MC Dropout or Deep Ensembles since it does not require multiple forward passes or model ensembles.}
\end{itemize}

\subsection{Theoretical Analysis of Uncertainty Calibration using Expected Calibration Errors}
%By calibrating the propensity uncertainty, the expected calibration error can be reduced, leading to improved CVR predictions. We now provide a theoretical analysis of the proposed method.

{The miscalibration of propensity scores, driven by overconfidence in both deep and shallow models, distorts IPS weights and hampers CVR predictions. By using Expected Calibration Error (ECE) to measure miscalibration, we demonstrate that reducing ECE improves propensity reliability and enhances prediction accuracy. By calibrating the propensity uncertainty, the Expected Calibration Error can be reduced, leading to improved CVR predictions. We now provide a theoretical analysis of the proposed method.}

We first derive the bias of the IPS estimator in Eqn.~\ref{eqn:IPS}:
\begin{lemma}
    \label{lemma:bias}
    Given inverse propensities of all user-item pairs $\hat{p}_{u,i}$, the bias of the IPS estimator in Eqn.~\ref{eqn:IPS} and the propensity bias are:
  \begin{eqnarray}
    \mathcal{E}_{\text{IPS}}=\left|\sum_{u,i\in\mathcal{D}}\frac{\nabla_{u,i}e_{u,i}}{|\mathcal{D}|}\right|, \\
        \nabla=\frac{\hat{p}_{u,i}-p_{u,i}}{\hat{p}_{u,i}}.
  \end{eqnarray}
\end{lemma}
Lemma \ref{lemma:bias}, as proved and cited from~\cite{wang2019doubly}, demonstrates that the bias of the IPS estimator is proportional to the biases in propensity scores. It follows directly that if the IPS estimator is well-calibrated, the bias term in Lemma~~\ref{lemma:bias} will be zero, indicating that a well-calibrated IPS estimator yields an unbiased estimate.

\begin{theorem}
    \label{theorem:upper_bound}
  For a calibrated IPS estimator, the bias is smaller than the uncalibrated IPS estimator:
    \begin{equation}
        \left|\sum_{u,i\in\mathcal{D}}\frac{\tilde{\nabla}_{u,i}e_{u,i}}{|\mathcal{D}|}\right|  
    \leq \left|\sum_{u,i\in\mathcal{D}}\frac{{\nabla}_{u,i}e_{u,i}}{|\mathcal{D}|}\right| ,
    \end{equation}
if the propensity is calibrated:
\begin{equation}
  \tilde{\nabla}_{u,i}=\frac{\tilde{p}_{u,i}-p_{u,i}}{\tilde{p}_{u,i}}\leq \nabla_{ui}, \tilde{p}=q(f(x)),
\end{equation}
where $q$ is a specific uncertainty calibration method, such as MC-Dropout, deep ensembles and the platt scaling.
\end{theorem}
\begin{proof}
  For a calibrated propensity, the propensity bias has a smaller bias and then the estimator bias smaller according to Lemma~\ref{lemma:bias}.
\end{proof}
Theorem~\ref{theorem:upper_bound} provides insights into the importance of uncertainty and Expected Calibration Error (ECE) in the Inverse Propensity Score (IPS) estimation. As demonstrated in the experiments, a significant reduction in the ECE of IPS leads to improved counterfactual recommendation results under Missing Not At Random (MNAR) conditions.

It has been rigorously analyzed in the literature that not only deep learning models but also shallow models, such as logistic regression, are inherently overconfident. The ECE of a well-specified logistic regression model is positive and cannot be completely eliminated. For further details, refer to~\cite{bai2021don,vaicenavicius2019evaluating}. Consequently, the original IPS estimator is also susceptible to miscalibrated uncertainty and large bias.

\begin{corollary}
  The unbiased and better calibration arguments in Theorem~\ref{theorem:upper_bound} also holds for the doubly robust estimator in~\cite{wang2019doubly}, which consists of the IPS estimator and the error-imputation-based estimator.
\end{corollary}
\begin{proof}
  It was shown in~\cite{wang2019doubly} that the bias term of the doubly estimator is also proportional to the IPS bias:
  \begin{equation}
    \mathcal{E}_{\text{IPS}}=\left|\sum_{u,i\in\mathcal{D}}\frac{\tilde{\nabla}_{u,i}\delta_{u,i}}{|\mathcal{D}|}\right|,
  \end{equation}
  where $\delta_{u,i}$ is the error derivation for missing ratings.
  This completes the proof.
\end{proof}
The prediction inaccuracy of a model is expected to be reduced through uncertainty calibration for the IPS estimator. Given the observed rating matrix $R$, the optimal rating prediction $\hat{R}^{*}$ is learned by the calibrated IPS estimator over the hypothesis space $\mathcal{H}$. We then present the generalization bound and the bias-variance decomposition of the calibrated IPS estimator using the Expected Calibration Errors~\cite{schnabel2016recommendations,wang2019doubly}.
\begin{theorem}\label{theorem:generalization-bound}
  For any finite hypothesis space $\mathcal{H}$ of the recommendation prediction estimations, the prediction error of the optimal prediction matrix $\hat{R}^{*}$ using the calibrated inverse propensity score estimator has the following generalization bound:
    \begin{equation}\label{eqn:generalization-bound}
       \mathcal{E}(\hat{R}^{*},R^{o})+
       \sum_{u,i\in\mathcal{D}}\frac{\tilde{\nabla}_{u,i}}{|\mathcal{D}|}+
       \sqrt{\frac{\log \frac{2|\mathcal{H}|}{\eta}}{2|\mathcal{D}|^2}\sum_{u,i\in\mathcal{D}}\frac{1}{\hat{p}_{ui}^2}},
    \end{equation}
    where the star superscript means the optimal prediction and the tilde means the calibrated IPS estimator.
    $R^{o}$ is the observed rating matrix $R^{o}=\{r_{ui}, o_{ui}=1\}$.
    The second term and third corresponds to the bias term and variance term respectively.
\end{theorem}
\begin{proof}
  Following the generalization bounds of the IPS and DR scoring models in \cite{schnabel2016recommendations,wang2019doubly}, we replace the propensity error with the calibrated one $\tilde{\nabla}$ and get the generalization bound of the calibrated IPS model.
\end{proof}
Theorem~\ref{theorem:generalization-bound} reveals the bias-variance tradeoff in the real-world performance of the calibrated inverse propensity score estimator. A smaller bias results from reduced propensity bias.

{Based theorem 2, better recommendation results come from only lower propensity estimation error but also Expected Calibration Error (ECE). Therein, lower ECE is the main goal of our uncertainty calibration algorithm since IPS predictions are generally overestimated. Therefore, with these two assumptions, we derive the following corollary.}

\begin{corollary}\label{corollary:bias-variance}
Compared with the inverse propensity score estimator, the prediction error bound of the calibrated doubly robust estimator has a smaller bias and has a upper bound that is proportional to $\text{ECE}$:
    \begin{equation}\label{eqn:ece-bound}
        \sum_{u,i\in\mathcal{D}}\frac{\tilde{\nabla}_{u,i}}{|\mathcal{D}|} \leq \sum_{u,i\in\mathcal{D}}\frac{n\cdot\text{ECE}}{|\mathcal{D}|},
    \end{equation}
    where $n$ is the number of the bins for ECE.
\end{corollary}
\begin{proof}
  The calibrated propensity has a lower bias so the bias term of the calibrated IPS is reduced:
  \begin{equation}
    \sum_{u,i\in\mathcal{D}}\frac{\tilde{\nabla}_{u,i}}{|\mathcal{D}|} < \sum_{u,i\in\mathcal{D}}\frac{{\nabla}_{u,i}}{|\mathcal{D}|}.
  \end{equation}
  For the upper bound that consists of ECE, we first rewrite ECE as:
  \begin{equation}
    \text{ECE}=\sum_{j=1}^{n}|\xi_{j}-\hat{\xi}_{j}|=\sum_{i=1}^{n}\left|\sum_{i=1}^{B_{j}}p_{ji}-\sum_{i=1}^{B_{j}}\tilde{p}_{ji}\right|,
  \end{equation}
  where $B_{ji}$ is the number of samples in the $j$-th bin and $p_{ji}$ is the propensity of the $i$-th sample in the $j$-th bin.
  By taking the absolute value for every bin , we can get the result of Eqn.~
  \ref{eqn:ece-bound}.
\end{proof}

%We defer the proofs to Appendix~\ref{appendix:bias-ips} and \ref{appendix:generalization-bound} due to the lack of space.

Corollary~\ref{corollary:bias-variance} demonstrates that the Expected Calibration Error (ECE) effectively bounds the final prediction error.  {
   Several existing works give proper theoretical anlaysis for the uncertainty calibration works, including parametric calibration function~\cite{vaicenavicius2019evaluating} and unparametric binning method~\cite{kumar2019verified}. These works reinforce the soundness of our approach and provide a broader theoretical context for uncertainty calibration.
}

\section{Experiments}
In this section, we will first provide an overview of the experimental setting, which includes details about the dataset, metrics, and baselines. We will then present our findings on calibration and CVR prediction based on two real-world datasets. Our experiments aim to address three key research questions (RQs):

\begin{enumerate}[(1)]
    \item To what extent is raw propensity estimation miscalibrated? How much improvement can be achieved through our uncertainty calibration module in terms of ECE?
    \item Why is ECE a reliable indicator of propensity score quality? Does a lower ECE result in increased CVR prediction task performance?
    \item How does uncertainty calibration enhance state-of-the-art models in terms of debiasing recommendation performance?
\end{enumerate}

\subsection{Experimental Setting}

\textbf{Datasets and Preprocessing} \\
To assess the debiasing capability of recommendation methods, it is crucial to have a Missing At Random (MAR) testing set. To achieve this, we utilize three prominent datasets: Coat Shopping, Yahoo! R3, and KuaiRand. These datasets contain MAR test sets that enable us to evaluate the performance of CVR prediction without bias~\cite{schnabel2016recommendations,wang2019doubly}.
\begin{itemize}
    \item \textbf{Coat Shopping}\footnote{\url{https://www.cs.cornell.edu/~schnabts/mnar/}}: The coat shopping dataset was collected to simulate the missing not at random data of user shopping for coats online. The customers were ordered to find their favorite coats in the online store. After browsing, the users were asked to rate 24 coats they had explored before and 16 randomly picked ones on a five point scale. It contains ratings from 290 users of 300 items. There are 6960 MNAR ratings  and 4640 MAR ratings. 
    \item \textbf{Yahoo! R3}\footnote{\url{http://webscope.sandbox.yahoo.com/}}: This dataset contains ratings for music collected from two different ways. The first source consists of ratings supplied by users during interaction with Yahoo Music services, which means that the data collected from this source suffer from Missing Not At Random problem. The second source consists of ratings to the music randomly recommended to users during an online survey which means that the data collected from this source is Missing Completely At Random. It includes approximately 300K ratings among which 54000 are MAR ratings.
    \item {\textbf{KuaiRand}\footnote{\url{https://kuairand.com/}}: this datasets includes 7583 videos and 27285 users, containing 1436609 biased data and 1186509 unbiased data. Following recent work\cite{song2023cdr}, we regard the biased data as the training set, and utilize the unbiased data for model validation (10\%) and evaluation (90\%).} 
\end{itemize}

To ensure consistency with the CVR prediction task, we preprocess the three datasets using methods established in previous studies~\cite{guo2021enhanced,saito2020doubly,dai2022generalized}. Here are the specific steps:
\begin{enumerate}[(1)]
    \item The conversion label $r_{u,i}$ is assigned a value of 1 if the rating for the user-item pair is greater than or equal to 4; otherwise it is assigned a value of 0. 
    \item Similarly, the click label $o_{u,i}$ is set to 1 if user $u$ has rated item $i$, and 0 otherwise.
    \item  We obtain the post-click conversion datasets as $\{(u,i,r_{u,i}) | o_{u,i} = 1, \forall (u,i) \in \mathcal{D}\}$
\end{enumerate}
Subsequently, we randomly split both datasets into training and validation sets. For MNAR ratings, 90\% of the ratings are allocated to the training set, while the remaining 10\% are reserved for the validation set. The MAR ratings are kept separate and used as a test set for evaluation purposes.

\noindent\textbf{Calibration Methods Settings} \\
We employ the aforementioned calibration methods for the uncalibrated inverse propensity score and select Neural Collaborative Filtering (\textit{NeuMF}) as the base recommendation method~\cite{he2017neural}. We denote the representative IPS models with and without calibration as follows:
\begin{itemize}
    \item \textbf{Raw Method:} We train a raw propensity estimation model that is not calibrated.
    \item \textbf{MC Dropout}~\cite{gal2016dropout}: Dropout is kept active during the inference stage. For a given user-item pair during testing, we pass it through the propensity model ten times with dropout active, averaging the results to obtain a calibrated propensity score.
    \item \textbf{Deep Ensembles}~\cite{lakshminarayanan2017simple}: We initialize ten models with different random seeds and shuffle the training dataset independently for each model. During testing, we aggregate predictions from these ten models and average the results to obtain calibrated propensity scores.
     {\item \textbf{Dual Focal Loss}\cite{tao2023dual}: We implement the Dual Focal Loss with a gamma parameter of 2.0 to train the propensity model. The loss function considers both the probability of the target label and the highest probability among all other labels that are smaller than the target probability, which helps to achieve better calibration.}
    \item \textbf{Platt Scaling}~\cite{platt1999probabilistic}: We optimize the cross-entropy loss using LBFGS to learn parameters \( b \) and \( c \) in Eqn.~\ref{eqn:platt} for calibrating the propensity scores. 
\end{itemize}

\noindent \textbf{Baselines}

We validate the effectiveness of our methods on three baselines, including two benchmark doubly robust (DR) methods, DR-JL~\cite{wang2019doubly} and MRDR~\cite{guo2021enhanced}, and one classical baseline, Inverse Propensity Scoring (IPS)~\cite{schnabel2016recommendations}. We also compare the calibrated and improved MRDR with four state-of-the-art methods: two are based on multi-task learning, ESCM$^2$-DR~\cite{wang2022escm} and DR-V2~\cite{pmlr-v202-li23ah}; the other two methods improve the propensity score estimation (GPL~\cite{10.1145/3583780.3614760}) and imputation error~\cite{song2023cdr}, respectively.

\noindent\textbf{Evaluation Metric}\\
For uncertainty calibration, we assess the Expected Calibration Error (ECE) using Eqn.~\ref{eqn:ece}. Other evaluation metrics include AUC,  discount cumulative gain (DCG) and Recall~\cite{guo2021enhanced}. Further implementation details, including evaluation metrics, optimization, and hyperparameters for all baselines, can be found in Section II in Supplemental Materials.

\subsection{Calibration Results of propensity scores(RQ1)}
We applied the three calibration methods for propensity scores and plotted the calibration curves along with the estimated Expected Calibration Error (ECE) for the propensity model. The ECE was computed using 100 bins.
\begin{table}[htbp]
\centering
\begin{tabular}{c|cc}
\toprule
\diagbox{Methods}{Datasets}    & Coat shopping & Yahoo! R3 \\ \midrule
raw         & 0.1458        & 0.1131    \\
MC Dropout     & 0.1369        & 0.1064    \\
Deep Ensemble    & 0.1408        & 0.1039    \\
Platt Scaling &\textbf{0.0433} & \textbf{0.0301}\\ 

\bottomrule
\end{tabular}
\caption{Expectation Calibration Errors of Calibrated Propensity Scores}
\label{tb:ece}
\end{table}

As shown in Table~\ref{tb:ece}, the calibration methods, especially Platt scaling, significantly reduce the Expected Calibration Error (ECE). Compared to uncalibrated propensity scores, Platt scaling reduces the ECE by more than a factor of three.
Figure~\ref{fig:coat} presents the calibration curves and propensity histograms of the calibrated propensity scores, where ``Raw" denotes uncalibrated propensity scores. In Figure~\ref{fig:coat}(a), calibration narrows the gap between the raw propensity model and the perfect propensity model (represented by the diagonal line). Since propensity-based debiasing methods rely solely on propensity scores from click events, the right side of the calibration curve further validates the effectiveness of propensity score calibration.

\begin{figure}[htbp!]
    \centering
    \subcaptionbox{Calibration Curve}{\includegraphics[width=0.23\textwidth]{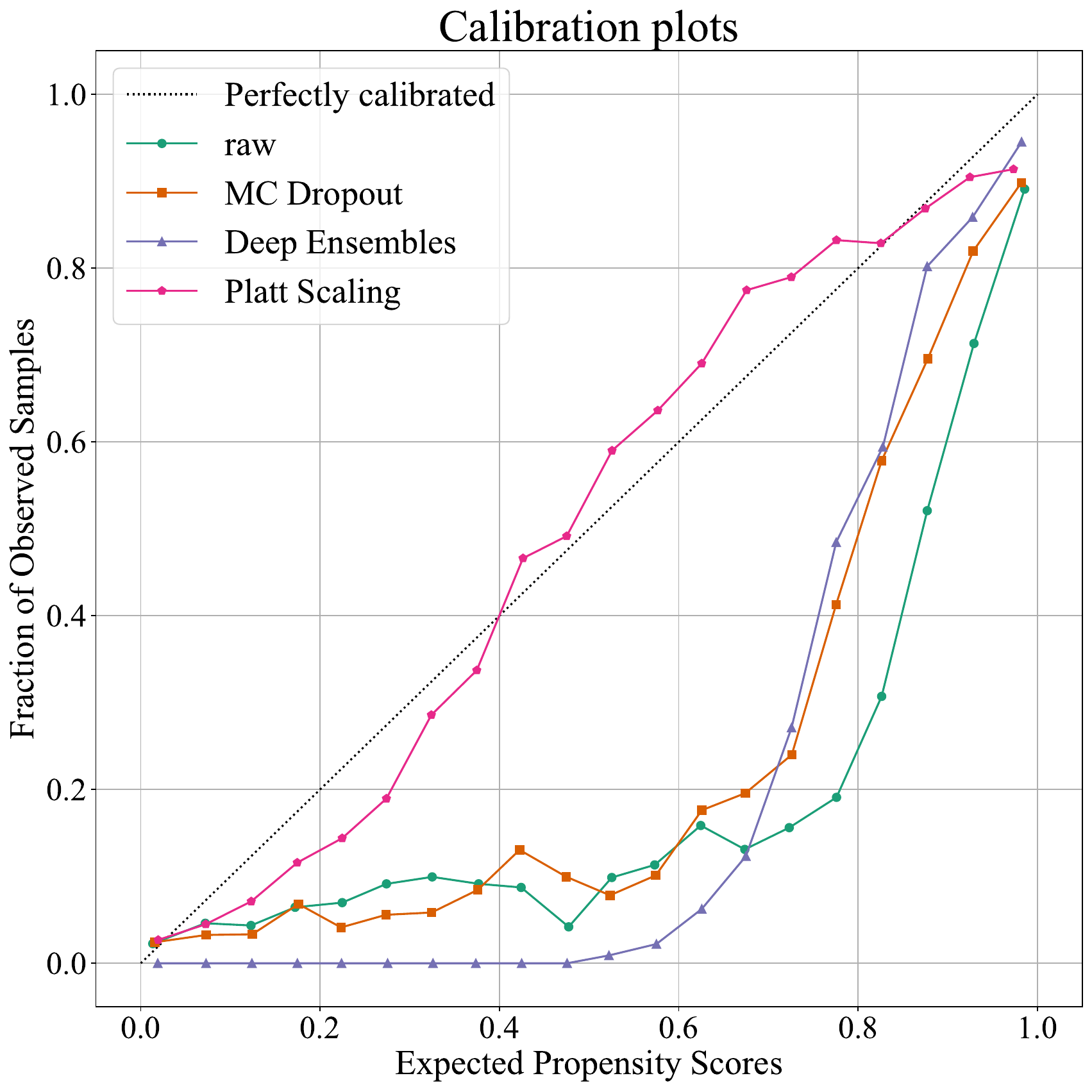}}
    \subcaptionbox{propensity scores Histgram}{\includegraphics[width=0.23\textwidth]{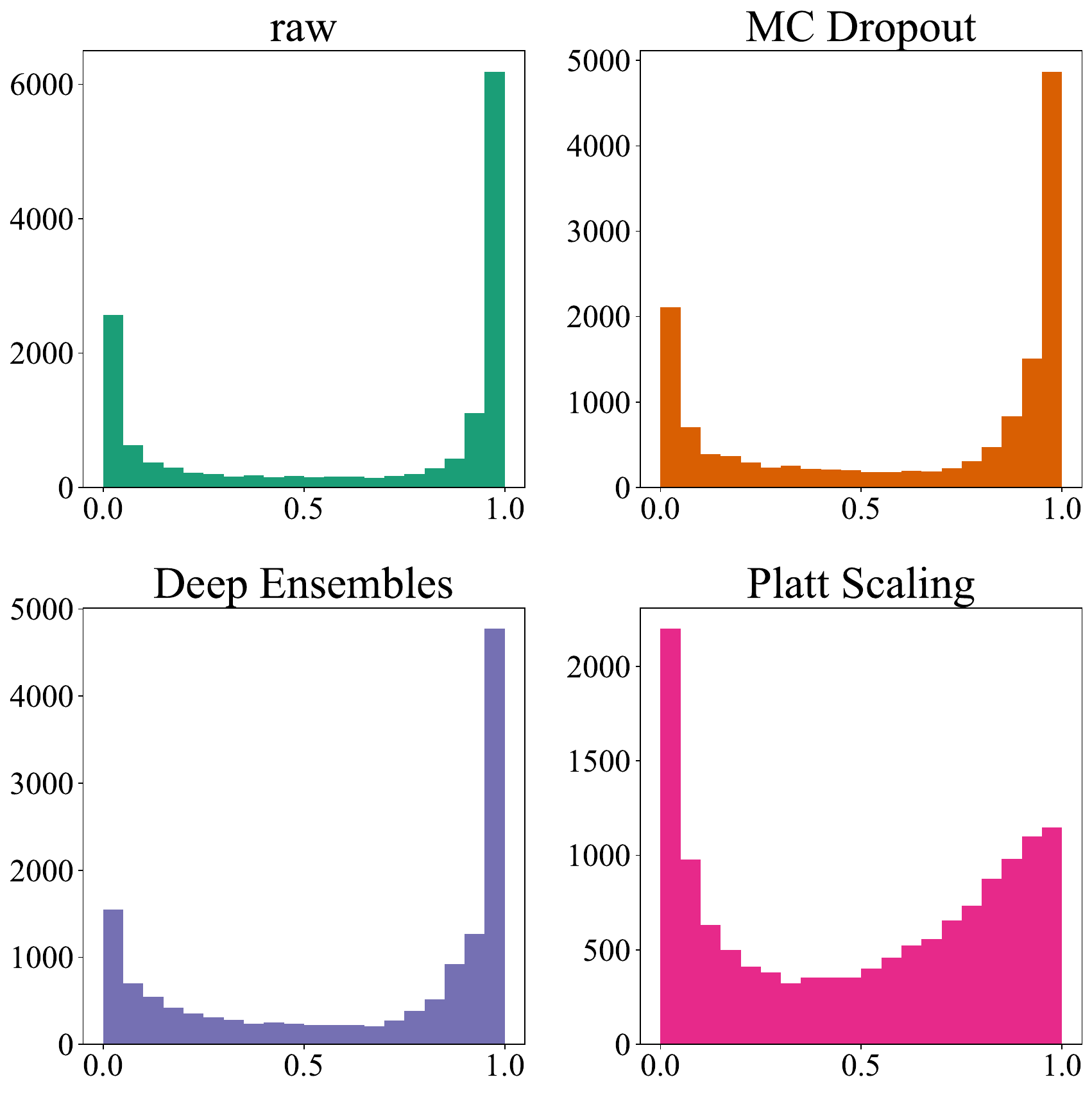}}
    \caption{Calibration Curve and Propensity Histograms of Calibrated propensity scores on the Coat Shopping Dataset}
    \label{fig:coat}
\end{figure}

\begin{table*}[!htbp]
\centering
\resizebox{\textwidth}{!}{
\begin{tabular}{llc cccccc}
\toprule
\multirow{2}{*}{Datasets}      & \multirow{2}{*}{Methods} &  \multirow{2}{*}{ {AUC}} & \multicolumn{3}{c}{DCG@K}                     & \multicolumn{3}{c}{Recall@K}          \\ \cmidrule(r){4-6} \cmidrule(r){7-9} 
&   & & K=2  & K=4   & K=6 & K=2 & K=4 & K=6           \\ \midrule\
\multirow{6}{*}{Coat Shopping} 
& ${\rm Neumf}_{base}$ &  {0.7604{\tiny $\pm$0.0041}} & 0.7478{\tiny $\pm$0.0201} & 1.0152{\tiny $\pm$0.0222} & 1.1989{\tiny $\pm$0.0243} & 0.8705{\tiny $\pm$0.0287} & 1.4435{\tiny $\pm$0.0301} & 1.9371{\tiny $\pm$0.0205} \\
& Raw &  {0.7578{\tiny $\pm$0.0036}} & 0.7472{\tiny $\pm$0.0143} & 1.0204{\tiny $\pm$0.0124} & 1.2010{\tiny $\pm$0.0111} & 0.8738{\tiny $\pm$0.0176} & 1.4591{\tiny $\pm$0.0177} & 1.9451{\tiny $\pm$0.0227} \\
& MC Dropout  &  {0.7632{\tiny $\pm$0.0011}‡} & \textbf{0.7651{\tiny $\pm$0.0242}‡} & \underline{1.0322{\tiny $\pm$0.0172}‡} & \underline{1.2101{\tiny $\pm$0.0145}‡} & \textbf{0.8962{\tiny $\pm$0.0283}‡} & \underline{1.4675{\tiny $\pm$0.0266}} & 1.9443{\tiny $\pm$0.0209} \\
& Deep Ensembles &  {\underline{0.7675{\tiny $\pm$0.0059}‡}} & 0.7584{\tiny$\pm$0.0271} & 1.0256{\tiny$\pm$0.0232} & 1.2065{\tiny$\pm$0.0225} & 0.8848{\tiny$\pm$0.0342} & 1.4574{\tiny $\pm$0.0341} & \underline{1.9454{\tiny $\pm$0.0322}} \\
&  {Dual FocalLoss} &  {0.7611{\tiny $\pm$0.0029}} & 0.7491{\tiny $\pm$0.0154} & 1.0264{\tiny $\pm$0.0129} & 1.2034{\tiny $\pm$0.0132} & 0.8751{\tiny $\pm$0.0253} & 1.4458{\tiny $\pm$0.0214} & 1.9479{\tiny $\pm$0.0299} \\
& Platt Scaling &  {\textbf{0.7693{\tiny $\pm$0.0036}‡}} & \underline{0.7627{\tiny $\pm$0.0155}‡} & \textbf{1.0405{\tiny $\pm$0.0215}‡} & \textbf{1.2259{\tiny $\pm$0.0171}‡} & \underline{0.8890{\tiny $\pm$0.0168}‡} & \textbf{1.4823{\tiny $\pm$0.0374}‡} & \textbf{1.9806{\tiny $\pm$0.0229}‡} \\
\midrule\midrule
\multirow{6}{*}{Yahoo! R3}     
& ${\rm Neumf}_{base}$ &  {0.7131{\tiny $\pm$0.0009}} & 0.5277{\tiny $\pm$0.0209} & 0.7352{\tiny $\pm$0.0209} & 0.8630{\tiny $\pm$0.0209} & 0.6333{\tiny $\pm$0.0209} & 1.0769{\tiny $\pm$0.0209} & 1.4202{\tiny $\pm$0.0209} \\
& Raw &  {0.7172{\tiny $\pm$0.0034}} & \underline{0.5433{\tiny $\pm$0.0056}} & 0.7395{\tiny $\pm$0.0065} & 0.8669{\tiny $\pm$0.0062} & \underline{0.6468{\tiny $\pm$0.0048}} & 1.0661{\tiny $\pm$0.0063} & 1.4086{\tiny $\pm$0.0057} \\
& MC Dropout  &  {0.7216{\tiny $\pm$0.0027}‡} & 0.5410{\tiny $\pm$0.0178} & 0.7406{\tiny $\pm$0.0202} & 0.8663{\tiny $\pm$0.0187} & 0.6452{\tiny $\pm$0.0197} & 1.0720{\tiny $\pm$0.0248}‡ & 1.4094{\tiny $\pm$0.0211} \\
& Deep Ensembles &  {\underline{0.7254{\tiny $\pm$0.0010}‡}} & 0.5342{\tiny $\pm$0.0043} & \underline{0.7412{\tiny $\pm$0.0047}} & \underline{0.8692{\tiny $\pm$0.0027}} & 0.6404{\tiny $\pm$0.0049} & \underline{1.0831{\tiny $\pm$0.0068}‡} & \underline{1.4270{\tiny $\pm$0.0047}‡} \\
&  {Dual FocalLoss} &  {0.7219{\tiny $\pm$0.0058}‡} &  {0.5441{\tiny $\pm$0.0114}} &  {0.7452{\tiny $\pm$0.0131}‡} &  {0.8720{\tiny $\pm$0.0120}‡} &  {0.6515{\tiny $\pm$0.0119}‡} &  {1.0708{\tiny $\pm$0.0147}‡} &  {1.4216{\tiny $\pm$0.0115}‡} \\
& Platt Scaling &  {\textbf{0.7235{\tiny $\pm$0.0017}‡}} & \textbf{0.5470{\tiny $\pm$0.0065}‡} & \textbf{0.7535{\tiny $\pm$0.0033}‡} & \textbf{0.8778{\tiny $\pm$0.0039}‡} & \textbf{0.6528{\tiny $\pm$0.0088}‡} & \textbf{1.0941{\tiny $\pm$0.0053}‡} & \textbf{1.4275{\tiny $\pm$0.0060}‡} \\
\bottomrule
\end{tabular}
}
\caption{Overall IPS-based recommendation performance on Coat Shopping and Yahoo! R3. The best results are shown in boldface, and the second best results are marked using underline. ‡ indicates statistically significant improvements over the Raw method at p $<$ 0.05 level.}
\label{tb:rec_results}
\end{table*}

Figure~\ref{fig:coat}(b) shows the propensity histograms of the calibrated IPS methods trained on the Coat Shopping dataset. It can be observed that the calibrated propensity scores not only exhibit lower ECE but also demonstrate reduced polarization. Both ECE and polarization are crucial aspects of propensity scores. The calibration curve and propensity histograms for the Yahoo! R3 dataset are detailed in Figure~2 in the Section IV in the Supplemental Material, which supports the findings from Figure~\ref{fig:coat}.

\subsection{CVR Prediction Results of IPS(RQ2)}

We utilize both uncalibrated and calibrated propensity scores to train the debiasing CVR models $f_\theta$, respectively. As a baseline, we train the recommendation model without using propensity scores, implying that all samples are given equal weight for loss. Table~\ref{tb:rec_results} presents the overall IPS debiasing performance in terms of DCG@K, Recall@K ($K=2,4,6$) and AUC on three real-world datasets\footnote{Recall refers to the recall number, which may exceed $1$.}. We repeat the experiments ten times and report the mean results to mitigate randomness. From the table, it can be observed that the IPS method with uncalibrated propensity scores achieves marginal improvement in recommendation performance compared to the baseline methods. Interestingly, the recall metric for the Yahoo! R3 dataset even shows a slight decrease, indicating that poorly calibrated propensity scores do not effectively aid the IPS-based training process.

With calibrated propensity scores, the IPS debiasing method shows significant improvement. As demonstrated in Table~\ref{tb:rec_results} and Table~\ref{tb:kuairand}, Platt scaling calibrated propensity scores outperform the uncalibrated ones in terms of all evaluation metrics on three real-world datasets. For instance, Platt scaling based IPS demonstrates substantial relative improvements of 1.51\%, 2.08\%, 1.98\%, and 2.07\% over the uncalibrated IPS method for AUC, DCG@2, DCG@4, and DCG@6 on the Coat Shopping dataset, respectively.

From the results presented in Table~\ref{tb:ece} and Table~\ref{tb:rec_results}, it is evident that propensity scores with lower calibration errors yield better recommendation results. The propensity scores calibrated by Platt scaling exhibit the lowest calibration error and outperform other calibration techniques across most recommendation evaluation metrics. Hence, ECE serves as a reliable measure of the effectiveness of propensity scores in mitigating bias in recommendations.

\subsection{CVR predcition Results of SOTA debiasing methods(RQ3)}

As our method improves the quality of propensity score estimation, it can readily be extended to other propensity score-based debiasing methods. We conducted experiments on six state-of-the-art CVR prediction models: DR-JL~\cite{wang2019doubly}, MRDR~\cite{guo2021enhanced}, GPL~\cite{zhou2023generalized}, CDR~\cite{song2023cdr}, DR-V2~\cite{pmlr-v202-li23ah}, and ESCM$^2$~\cite{wang2022escm}. Table~\ref{tb:dr} demonstrates that calibrated propensity scores outperform raw propensity scores on all evaluation metrics, highlighting the effectiveness of uncertainty calibration for other propensity score-based methods. Table~\ref{tb:sota} shows that our method surpasses current state-of-the-art (SOTA) approaches in terms of performance. Table~\ref{tb:sota} demonstrates that our method outperforms the current state-of-the-art (SOTA) approaches in terms of performance.  {This improvement is largely attributed to our method’s ability, particularly when combined with Platt scaling calibration, to more accurately estimate the probability of each (user, item) pair appearing in the observed data. Enhanced accuracy in propensity score estimation directly translates to higher-quality recommendations}

 {Experimental results consistently confirm that better-calibrated propensity scores lead to superior performance. As illustrated in Tables II–V, methods utilizing calibrated propensity scores consistently outperform the raw baseline. This highlights a strong correlation between calibration quality and recommendation performance. Notably, Platt scaling achieves the highest calibration accuracy, which directly results in the best recommendation outcomes. In contrast, the raw baseline suffers from biased propensity scores that fail to accurately represent true interaction probabilities, leading to suboptimal recommendation results. Our experiments clearly show that this limitation can be effectively mitigated through proper calibration, underscoring the critical role of accurate propensity score estimation in achieving superior recommendation quality.}

\subsection{Efficiency experiment}

\begin{table}[h]
\begin{tabular}{ccccc}
\toprule
Methods  & Raw  & Platt Scaling & MC Dropout & Deep Ensembles    \\  \midrule
Training & 34.12s  & 36.31s         & 34.12s     & 246.34s                   \\
Inference     & 2.51s & 2.53s          & 5.78s      & 6.11s                    \\ \bottomrule
\end{tabular}
\caption{The time consumption of the Propensity estimation model employing different calibration techniques on dataset Coat Shopping in seconds. }
\end{table}

Table 4 presents the time consumption of the propensity estimation model using different calibration techniques on the Coat Shopping dataset. The efficiency experiment was conducted on a single 3090 GPU. It is evident that both Platt scaling and MC dropout techniques exhibit low time costs, whereas Deep Ensemble incurs higher costs due to the necessity of training multiple models. Therefore, the BatchEnsembles method~\cite{wen2019batchensemble} can be employed to reduce the overall computation cost. These efficiency experiment results are consistent with the complexity analysis in Section~\ref{sec:computational complexity}.

\begin{table*}[!htbp]
\centering
\resizebox{\textwidth}{!}{
\begin{tabular}{c|ccccccccc}
\toprule
\multicolumn{1}{c}{}& & & \multirow{2}{*}{ {AUC}} & \multicolumn{3}{c}{DCG@K} & \multicolumn{3}{c}{Recall@K} \\ \cmidrule(r){5-7} \cmidrule(r){8-10} 
\multicolumn{1}{c}{\multirow{-2}{*}{Baseline}} & \multirow{-2}{*}{Dataset}  & \multirow{-2}{*}{Methods} & & K=2 & K=4  & K=6 & K=2  & K=4 & K=6 \\ \midrule

& & Raw &  {0.7644{\tiny $\pm$0.0032}} & 0.7454{\tiny $\pm$0.0158} & 1.0185{\tiny $\pm$0.0196} & 1.2014{\tiny $\pm$0.0126} & 0.8717{\tiny $\pm$0.0214} & 1.4570{\tiny $\pm$0.0243} & 1.9489{\tiny $\pm$0.0209} \\

& & Dropout  &  {0.7673{\tiny $\pm$0.0024}} & 0.7496{\tiny $\pm$0.0205} & \underline{1.0271{\tiny $\pm$0.0165}‡} & 1.2073{\tiny $\pm$0.0157} & \underline{0.8835{\tiny $\pm$0.0249}‡} & \underline{1.4764{\tiny $\pm$0.0216}‡} & 1.9603{\tiny $\pm$0.0286} \\

& & Ensembles &  {\textbf{0.7805{\tiny $\pm$0.0028}}‡} & \textbf{0.7546{\tiny $\pm$0.0160}‡} & 1.0254{\tiny $\pm$0.0110} & \underline{1.2132{\tiny $\pm$0.0119}‡} & \underline{0.8835{\tiny $\pm$0.0228}‡} & 1.4637{\tiny $\pm$0.0126} & \underline{1.9684{\tiny $\pm$0.0232}‡} \\

& &  {Dual FocalLoss} &  {0.7699{\tiny $\pm$0.0030}‡} &  {\underline{0.7542{\tiny $\pm$0.0163}}‡} &  {1.0106{\tiny $\pm$0.0216}} &  {1.1939{\tiny $\pm$0.0179}} &  {0.8834{\tiny $\pm$0.0154}‡} &  {1.4651{\tiny $\pm$0.0257}} &  {1.9514{\tiny $\pm$0.0210}} \\

& \multirow{-5}{*}{Coat Shopping} & Platt  &  {\underline{0.7730{\tiny $\pm$0.0018}}‡} & 0.7539{\tiny $\pm$0.0201}‡ & \textbf{1.0389{\tiny $\pm$0.0208}‡} & \textbf{1.2241{\tiny $\pm$0.0170}‡} & \textbf{0.8852{\tiny $\pm$0.0235}‡} & \textbf{1.4940{\tiny $\pm$0.0290}‡} & \textbf{1.9928{\tiny $\pm$0.0208}‡} \\

\cmidrule(r){2-10} 

& & Raw &  {0.7152{\tiny $\pm$0.0053}} & 0.5450{\tiny $\pm$0.0093} & 0.7405{\tiny $\pm$0.0097} & 0.8779{\tiny $\pm$0.0091} & 0.6402{\tiny $\pm$0.0113} & 1.0587{\tiny $\pm$0.0143} & 1.4112{\tiny $\pm$0.0141} \\

& & Dropout &  {0.7130{\tiny $\pm$0.0049}} & \underline{0.5501{\tiny $\pm$0.0152}} & \underline{0.7597{\tiny $\pm$0.0126}‡} & 0.8774{\tiny $\pm$0.0103} & \underline{0.6539{\tiny $\pm$0.0161}‡} & 1.0822{\tiny $\pm$0.0124}‡ & 1.4183{\tiny $\pm$0.0065}‡ \\

& & Ensembles &  {\textbf{0.7258{\tiny $\pm$0.0011}}‡} & 0.5431{\tiny $\pm$0.0053} & 0.7491{\tiny $\pm$0.0043}‡ & 0.8747{\tiny $\pm$0.0050} & 0.6510{\tiny $\pm$0.0044}‡ & \underline{1.0923{\tiny $\pm$0.0054}‡} & \underline{1.4293{\tiny $\pm$0.0068}‡} \\

& &  {Dual FocalLoss} &  {0.7121{\tiny $\pm$0.0141}} &  {0.5358{\tiny $\pm$0.0118}} &  {0.7447{\tiny $\pm$0.0114}} & \underline{ {0.8788{\tiny $\pm$0.0295}}} &  {0.6429{\tiny $\pm$0.0278}} &  {1.0584{\tiny $\pm$0.0287}} &  {1.4186{\tiny $\pm$0.0276}‡} \\

\multirow{-10}{*}{DR-JL} & \multirow{-5}{*}{Yahoo! R3} & Platt  &  {\underline{0.7248{\tiny $\pm$0.0016}}‡} & \textbf{0.5532{\tiny $\pm$0.0058}‡} & \textbf{0.7555{\tiny $\pm$0.0042}‡} & \textbf{0.8816{\tiny $\pm$0.0034}‡} & \textbf{0.6602{\tiny $\pm$0.0059}‡} & \textbf{1.0926{\tiny $\pm$0.0073}‡} & \textbf{1.4314{\tiny $\pm$0.0062}‡} \\

\midrule \midrule 

& & Raw &  {\underline{0.7691{\tiny $\pm$0.0035}}} & 0.6830{\tiny $\pm$0.0151} &  0.9661{\tiny $\pm$0.0131} &  1.1578{\tiny $\pm$0.0101} &  0.8185{\tiny $\pm$0.0163} &  1.4261{\tiny $\pm$0.0192} &  1.9409{\tiny $\pm$0.0205} \\

& & Dropout &  {0.7661{\tiny $\pm$0.0011}} & 0.7255{\tiny $\pm$0.0200}‡ & 0.9946{\tiny $\pm$0.0138}‡ & 1.1847{\tiny $\pm$0.0127}‡ & 0.8438{\tiny $\pm$0.0257}‡ & 1.4177{\tiny $\pm$0.0247} & 1.9282{\tiny $\pm$0.0312} \\

& & Ensembles &  {0.7647{\tiny $\pm$0.0038}} & \underline{0.7292{\tiny $\pm$0.0273}‡} & 1.0001{\tiny $\pm$0.0192}‡ & 1.1846{\tiny $\pm$0.0182}‡ & \underline{0.8523{\tiny $\pm$0.0352}‡} & 1.4303{\tiny $\pm$0.0190} & 1.9282{\tiny $\pm$0.0299} \\

& &  {Dual FocalLoss} &  {0.7606{\tiny $\pm$0.0052}} &  {0.7260{\tiny $\pm$0.0205}‡} &  {\underline{1.0041{\tiny $\pm$0.0162}}‡} &  {\underline{1.1885{\tiny $\pm$0.0203}}‡} &  {0.8439{\tiny $\pm$0.0222}‡} & \underline{ {1.4395{\tiny $\pm$0.0171}‡}} & \underline{ {1.9353{\tiny $\pm$0.0225}}} \\

& \multirow{-5}{*}{Coat Shopping} & Platt  &  {\textbf{0.7728{\tiny $\pm$0.0025}}‡} & \textbf{0.7648{\tiny $\pm$0.0190}‡} & \textbf{1.0155{\tiny $\pm$0.0170}‡} & \textbf{1.2223{\tiny $\pm$0.0165}‡} & \textbf{0.8987{\tiny $\pm$0.0230}‡} & \textbf{1.4688{\tiny $\pm$0.0258}‡} & \textbf{1.9957{\tiny $\pm$0.0395}‡} \\

\cmidrule(r){2-10} 

& & Raw &  {0.6678{\tiny $\pm$0.0162}} & 0.5371{\tiny $\pm$0.0447} & 0.7441{\tiny $\pm$0.0502} & 0.8636{\tiny $\pm$0.0455} & 0.6383{\tiny $\pm$0.0508} & 1.0808{\tiny $\pm$0.0638} & 1.4006{\tiny $\pm$0.0524} \\

& & Dropout &  {0.6671{\tiny $\pm$0.0135}} & \underline{0.5458{\tiny $\pm$0.0194}} & \underline{0.7461{\tiny $\pm$0.0240}} & 0.8718{\tiny $\pm$0.0225}‡ & \underline{0.6547{\tiny $\pm$0.0209}‡} & 1.0821{\tiny $\pm$0.0363} & 1.4199{\tiny $\pm$0.0382}‡ \\

& & Ensembles &  {\underline{0.6907{\tiny $\pm$0.0026}}‡} & 0.5410{\tiny $\pm$0.0093} &  0.7443{\tiny $\pm$0.0098} &  \underline{0.8731{\tiny $\pm$0.0090}‡} &  0.6444{\tiny $\pm$0.0120} &  1.0809{\tiny $\pm$0.0133} &  \underline{1.4256{\tiny $\pm$0.0127}‡} \\

& &  {Dual FocalLoss} &  {0.6679{\tiny $\pm$0.0461}} &  {0.5252{\tiny $\pm$0.0062}} &  {0.7460{\tiny $\pm$0.0194}} &  {0.8640{\tiny $\pm$0.0203}} &  {0.6445{\tiny $\pm$0.0191}} & \underline{ {1.0829{\tiny $\pm$0.0166}}} &  {1.4144{\tiny $\pm$0.0190}‡} \\

\multirow{-10}{*}{MRDR} & \multirow{-5}{*}{Yahoo! R3} & Platt &  {\textbf{0.6988{\tiny $\pm$0.0018}}‡} & \textbf{0.5623{\tiny $\pm$0.0092}‡} & \textbf{0.7571{\tiny $\pm$0.0066}‡} & \textbf{0.8858{\tiny $\pm$0.0072}‡} & \textbf{0.6687{\tiny $\pm$0.0103}‡} & \textbf{1.0862{\tiny $\pm$0.0080}‡} & \textbf{1.4326{\tiny $\pm$0.009}‡} \\

\bottomrule
\end{tabular}
}
\caption{DRJL and MRDR CVR prediction performance on Coat Shopping and Yahoo! R3. The best results are shown in boldface and the second best results are marked using underline. ‡ indicates statistically significant improvements over the Raw method at p $<$ 0.05 level.}
\label{tb:dr}
\end{table*}

% \usepackage{multirow}
% \usepackage[normalem]{ulem}
\begin{table*}[!htbp]
\centering
\resizebox{\textwidth}{!}{
\begin{tabular}{cccccccccc}
\toprule
\multicolumn{2}{c}{}                          & \multirow{2}{*}{ {AUC}} & \multicolumn{3}{c}{DCG@K}                        & \multicolumn{3}{c}{Recall@K}                     &                           \\ \cmidrule(r){4-6} \cmidrule(r){7-9}
\multicolumn{2}{c}{\multirow{-2}{*}{Methods}} & & K=2                    & K=4                    & K=6                    & K=2                    & K=4                    & K=6                    & \multirow{-2}{*}{Average} \\ \midrule

& Raw              &  {0.6493{\tiny $\pm$0.0034}} & 0.4426{\tiny $\pm$ 0.0076} & 0.6728{\tiny $\pm$ 0.0110} & 0.8471{\tiny $\pm$ 0.0108} & 0.5404{\tiny $\pm$ 0.0096} & 1.0351{\tiny $\pm$ 0.0166} & 1.5028{\tiny $\pm$ 0.0163} & 0.8401 \\
& MC Dropout       &  {0.6366{\tiny $\pm$0.0058}} & 0.4515{\tiny $\pm$ 0.0063}‡ & 0.6855{\tiny $\pm$ 0.0064}‡ & 0.8593{\tiny $\pm$ 0.0069}‡ & 0.5520{\tiny $\pm$ 0.0073}‡ & 1.0545{\tiny $\pm$ 0.0080}‡ & 1.5208{\tiny $\pm$ 0.0101}‡ & 0.8539 \\
& Deep Ensembles   & \underline{ {0.6550{\tiny $\pm$0.0019}}} & 0.4562{\tiny $\pm$ 0.0045}‡ & {0.6893{\tiny $\pm$ 0.0048}‡} & {0.8627{\tiny $\pm$ 0.0042}‡} & {0.5579{\tiny $\pm$ 0.0052}‡} & {1.0575{\tiny $\pm$ 0.0063}‡} & {1.5228{\tiny $\pm$ 0.0058}‡} & {0.8577} \\
&  {Dual FocalLoss} &  {0.6373{\tiny $\pm$0.0076}} & \underline{ {0.4634{\tiny $\pm$0.0089}‡}} & \textbf{ {0.7068{\tiny $\pm$0.0116}‡}} & \underline{ {0.8705{\tiny $\pm$0.0130}‡}} & \underline{ {0.5648{\tiny $\pm$0.0106}‡}} & \textbf{ {1.0761{\tiny $\pm$0.0161}‡}} & \underline{ {1.5420{\tiny $\pm$0.0200}‡}} & \underline{ {0.8706}}\\
\multirow{-5}{*}{IPS} & Platt Scaling    & \textbf{ {0.6682{\tiny $\pm$0.0029}‡}} & \textbf{0.4657{\tiny $\pm$ 0.0030}‡} & \underline{0.7021{\tiny $\pm$ 0.0025}‡} & \textbf{0.8774{\tiny $\pm$ 0.0026}‡} & \textbf{0.5677{\tiny $\pm$ 0.0029}‡} & \textbf{1.0754{\tiny $\pm$ 0.0030}‡} & \textbf{1.5455{\tiny $\pm$ 0.0036}‡} & \textbf{0.8723} \\

\midrule

& Raw              &  {0.6478{\tiny $\pm$0.0022}} & 0.4442{\tiny $\pm$ 0.0083} & 0.6742{\tiny $\pm$ 0.0111} & 0.8481{\tiny $\pm$ 0.0115} & 0.5420{\tiny $\pm$ 0.0096} & 1.0362{\tiny $\pm$ 0.0160} & 1.5026{\tiny $\pm$ 0.0170} & 0.8412 \\
& MC Dropout       &  {0.6343{\tiny $\pm$0.0074}} & 0.4504{\tiny $\pm$ 0.0042}‡ & 0.6839{\tiny $\pm$ 0.0044}‡ & 0.8580{\tiny $\pm$ 0.0052}‡ & 0.5504{\tiny $\pm$ 0.0041}‡ & 1.0520{\tiny $\pm$ 0.0047}‡ & 1.5189{\tiny $\pm$ 0.0071}‡ & 0.8523 \\
& Deep Ensembles   &  {\underline{0.6547{\tiny $\pm$0.0030}}} & {0.4524{\tiny $\pm$ 0.0064}‡} & {0.6854{\tiny $\pm$ 0.0039}‡} & {0.8606{\tiny $\pm$ 0.0043}‡} & {0.5528{\tiny $\pm$ 0.0051}‡} & {1.0530{\tiny $\pm$ 0.0048}‡} & {1.5231{\tiny $\pm$ 0.0067}‡} & {0.8546} \\
&  {Dual FocalLoss} &  {0.6436{\tiny $\pm$0.0185}} & {\underline{0.4673}{\tiny $\pm$0.0410}‡} &  {\underline{0.7065}{\tiny $\pm$0.0460}‡} &  {\underline{0.8815}{\tiny $\pm$0.0459}‡} &  {\underline{0.5698}{\tiny $\pm$0.0458}‡} &  {\textbf{1.0834}{\tiny $\pm$0.0565}‡} &  {\underline{1.5530}{\tiny $\pm$0.0564}‡} &  {\underline{0.8769}} \\

\multirow{-5}{*}{DR-JL} & Platt Scaling    &  {\textbf{0.6679{\tiny $\pm$0.0017}‡}} & \textbf{0.4701{\tiny $\pm$ 0.0073}‡} & \textbf{0.7070{\tiny $\pm$ 0.0079}‡} & \textbf{0.8834{\tiny $\pm$ 0.0080}‡} & \textbf{0.5733{\tiny $\pm$ 0.0081}‡} & \underline{1.0823{\tiny $\pm$ 0.0095}‡} & \textbf{1.5553{\tiny $\pm$ 0.0096}‡} & \textbf{0.8786} \\

\midrule
& Raw              &  {0.5465{\tiny $\pm$0.0064}} & 0.4369{\tiny $\pm$ 0.0246} & 0.6531{\tiny $\pm$ 0.0287} & 0.8144{\tiny $\pm$ 0.0295} & 0.5305{\tiny $\pm$ 0.0292} & 0.9957{\tiny $\pm$ 0.0376} & 1.4290{\tiny $\pm$ 0.0401} & 0.8099 \\
& MC Dropout       &  {0.5491{\tiny $\pm$0.0044}} & {0.4412{\tiny $\pm$ 0.0125}‡} & 0.6659{\tiny $\pm$ 0.0183}‡ & 0.8305{\tiny $\pm$ 0.0200}‡ & 0.5376{\tiny $\pm$ 0.0152}‡ & 1.0205{\tiny $\pm$ 0.0275}‡ & 1.4630{\tiny $\pm$ 0.0302}‡ & 0.8265 \\
& Deep Ensembles   &  {\underline{0.5855{\tiny $\pm$0.0050}‡}} & 0.4406{\tiny $\pm$ 0.0101}‡ & {0.6681{\tiny $\pm$ 0.0127}‡} & {0.8432{\tiny $\pm$ 0.0135}‡} & {0.5390{\tiny $\pm$ 0.0131}‡} & {1.0275{\tiny $\pm$ 0.0171}‡} & {1.4972{\tiny $\pm$ 0.0215}‡} & {0.8359} \\
&  {Dual FocalLoss} &  {0.5813{\tiny $\pm$0.0028}‡} &  {\underline{0.4777{\tiny $\pm$0.0151}‡}} &  {\underline{0.7165{\tiny $\pm$0.0169}‡}} &  {\underline{0.8920{\tiny $\pm$0.0183}‡}} &  {\underline{0.5828{\tiny $\pm$0.0173}‡}} &  {\underline{1.0955{\tiny $\pm$0.0217}‡}} &  {\textbf{1.5661{\tiny $\pm$0.0273}‡}} &  {\underline{0.8884}} \\
\multirow{-5}{*}{MRDR} & Platt Scaling    &  {\textbf{0.6247{\tiny $\pm$0.0035}‡}} &\textbf{0.4928{\tiny $\pm$ 0.0057}‡} & \textbf{0.7259{\tiny $\pm$ 0.0076}‡} & \textbf{0.8963{\tiny $\pm$ 0.0079}‡} & \textbf{0.5971{\tiny $\pm$ 0.0070}‡} & \textbf{1.0972{\tiny $\pm$ 0.0109}‡} & \underline{1.5549{\tiny $\pm$ 0.0119}‡} & \textbf{0.8940}\\

\bottomrule
\end{tabular}
}
\caption{Overall performance on KuaiRand. ‡ indicates statistically significant improvements over the Raw method at p $<$0.05 level.}
\label{tb:kuairand}
\end{table*}

% \usepackage{multirow}
% \usepackage[normalem]{ulem}
\begin{table*}[]
\centering
\resizebox{\textwidth}{!}{
\begin{tabular}{c|cccccccc}
\toprule
                     &           & \multirow{2}{*}{ {AUC}} & \multicolumn{3}{c}{DCG@K} & \multicolumn{3}{c}{Recall@K} \\ \cmidrule(r){4-6}\cmidrule(r){7-9}
\multirow{-2}{*}{Datasets}      & \multirow{-2}{*}{Methods} & & K=2 & K=4 & K=6 & K=2 & K=4 & K=6 \\ \midrule
                                & MRDR-GPL &  {0.7521{\tiny $\pm$0.0035}} & 0.7488{\tiny $\pm$ 0.0201} & 1.0061{\tiny $\pm$ 0.0222} & 1.1949{\tiny $\pm$ 0.0243} & 0.8734{\tiny $\pm$ 0.0287} & 1.4219{\tiny $\pm$ 0.0301} & 1.9283{\tiny $\pm$ 0.0405} \\
                                & MRDR-CDR &  {0.7622{\tiny $\pm$0.0031}} & 0.7579{\tiny $\pm$ 0.0201} & \underline{1.0192{\tiny $\pm$ 0.0192}} & 1.1991{\tiny $\pm$ 0.0198} & 0.8903{\tiny $\pm$ 0.0204} & 1.4515{\tiny $\pm$ 0.0277} & 1.9325{\tiny $\pm$ 0.0402} \\
                                & DR-V2 &  {0.7637{\tiny $\pm$0.0039}} & \textbf{0.7746{\tiny $\pm$ 0.0188}} & 1.0116{\tiny $\pm$ 0.0164} & 1.2076{\tiny $\pm$ 0.0176} & 0.8841{\tiny $\pm$ 0.0184} & 1.4568{\tiny $\pm$ 0.0152} & 1.9494{\tiny $\pm$ 0.0324} \\
                                & ESCM$^2$-DR & \underline{ {0.7681{\tiny $\pm$0.0041}}} & 0.7528{\tiny $\pm$ 0.0177} & \textbf{1.0273{\tiny $\pm$ 0.0189}} & \underline{1.2081{\tiny $\pm$ 0.0229}} & 0.8945{\tiny $\pm$ 0.0186} & \textbf{1.4810{\tiny $\pm$ 0.0162}} & 1.9662{\tiny $\pm$ 0.0299} \\ \rowcolor{lightgray}
\cellcolor{white}\multirow{-5}{*}{Coat Shopping} & MRDR-CAL(Ours) & \textbf{ {0.7728{\tiny $\pm$0.0025}‡}} & \underline{0.7648{\tiny $\pm$ 0.0190}‡} & 1.0155{\tiny $\pm$ 0.0170} & \textbf{1.2223{\tiny $\pm$ 0.0165}‡} & \textbf{0.8987{\tiny $\pm$ 0.0230}} & \underline{1.4688{\tiny $\pm$ 0.0258}} & \textbf{1.9957{\tiny $\pm$ 0.0395}‡} \\ \midrule
                                & MRDR-GPL &  {0.6617{\tiny $\pm$0.0064}} & 0.5384{\tiny $\pm$ 0.0194} & 0.7369{\tiny $\pm$ 0.0252} & 0.8605{\tiny $\pm$ 0.0211} & 0.6408{\tiny $\pm$ 0.0197} & 1.0657{\tiny $\pm$ 0.0222} & 1.3982{\tiny $\pm$ 0.0241} \\
                                & MRDR-CDR &  {0.6673{\tiny $\pm$0.0035}} & 0.5417{\tiny $\pm$ 0.0162} & 0.7456{\tiny $\pm$ 0.0123} & 0.8698{\tiny $\pm$ 0.0125} & 0.6490{\tiny $\pm$ 0.0202} & \underline{1.0842{\tiny $\pm$ 0.0182}} & 1.4175{\tiny $\pm$ 0.0147} \\
                                & DR-V2 &  {0.6807{\tiny $\pm$0.0026}} & 0.5518{\tiny $\pm$ 0.0125} & 0.7479{\tiny $\pm$ 0.0143} & 0.8732{\tiny $\pm$ 0.0156} & \underline{0.658{\tiny $\pm$ 0.0111}} & 1.0784{\tiny $\pm$ 0.0181} & 1.4154{\tiny $\pm$ 0.0158} \\
                                & ESCM$^2$-DR & \underline{ {0.6828{\tiny $\pm$0.0161}}} & \underline{0.5541{\tiny $\pm$ 0.0144}} & \underline{0.7502{\tiny $\pm$ 0.0126}} & \underline{0.8771{\tiny $\pm$ 0.0171}} & 0.6564{\tiny $\pm$ 0.0156} & 1.0772{\tiny $\pm$ 0.0175} & \underline{1.4183{\tiny $\pm$ 0.0154}} \\ \rowcolor{lightgray}
\cellcolor{white}\multirow{-5}{*}{Yahoo! R3} & MRDR-CAL(Ours) & \textbf{ {0.6988{\tiny $\pm$0.0018}‡}} & \textbf{0.5623{\tiny $\pm$ 0.0092}‡} & \textbf{0.7571{\tiny $\pm$ 0.0066}‡} & \textbf{0.8858{\tiny $\pm$ 0.0072}‡} & \textbf{0.6687{\tiny $\pm$ 0.0103}‡} & \textbf{1.0862{\tiny $\pm$ 0.0080}‡} & \textbf{1.4326{\tiny $\pm$ 0.0090}‡} \\ \midrule
                                & MRDR-GPL &  {0.5401{\tiny $\pm$0.0084}} & 0.4335{\tiny $\pm$ 0.0123} & 0.6446{\tiny $\pm$ 0.0178} & 0.8049{\tiny $\pm$ 0.0143} & 0.5261{\tiny $\pm$ 0.0171} & 0.9792{\tiny $\pm$ 0.0178} & 1.4101{\tiny $\pm$ 0.0146} \\
                                & MRDR-CDR &  {0.5498{\tiny $\pm$0.0054}} & 0.4326{\tiny $\pm$ 0.0098} & 0.6573{\tiny $\pm$ 0.0145} & 0.8297{\tiny $\pm$ 0.0132} & 0.5289{\tiny $\pm$ 0.0126} & 1.0111{\tiny $\pm$ 0.0146} & 1.4739{\tiny $\pm$ 0.0165} \\
                                & DR-V2 &  {0.5765{\tiny $\pm$0.0047}} & 0.4465{\tiny $\pm$ 0.0078} & 0.6795{\tiny $\pm$ 0.0121} & 0.8533{\tiny $\pm$ 0.0098} & 0.5452{\tiny $\pm$ 0.0100} & 1.0459{\tiny $\pm$ 0.0122} & 1.5117{\tiny $\pm$ 0.0123} \\
                                & ESCM$^2$-DR & \underline{ {0.5836{\tiny $\pm$0.0068}}} & \underline{0.4773{\tiny $\pm$ 0.0101}} & \underline{0.7059{\tiny $\pm$ 0.0155}} & \underline{0.8760{\tiny $\pm$ 0.0121}} & \underline{0.5781{\tiny $\pm$ 0.0143}} & \underline{1.0694{\tiny $\pm$ 0.0143}} & \underline{1.5258{\tiny $\pm$ 0.0152}} \\ \rowcolor{lightgray}
\cellcolor{white}\multirow{-5}{*}{KuaiRand} & MRDR-CAL(Ours) & \textbf{ {0.6247{\tiny $\pm$0.0035}‡}} & \textbf{0.4928{\tiny $\pm$ 0.0057}‡} & \textbf{0.7259{\tiny $\pm$ 0.0076}‡} & \textbf{0.8963{\tiny $\pm$ 0.0079}‡} & \textbf{0.5971{\tiny $\pm$ 0.0070}‡} & \textbf{1.0972{\tiny $\pm$ 0.0109}‡} & \textbf{1.5549{\tiny $\pm$ 0.0119}‡} \\ \bottomrule
\end{tabular}}
\caption{The comparison with the SOTA methods, ‡ indicates statistically significant improvements over the ESCM$^2$-DR method at p $<$ 0.05 level.}
\label{tb:sota}
\end{table*}

\section{Related Works}

\subsection{Approaches to CVR Estimation}

In practical applications, CTR prediction models are often adapted for CVR prediction tasks due to their conceptual similarities. These approaches encompass various methods, including logistic regression-based models \cite{Richardson2007PredictingCE}, factorization machine-based models \cite{Juan2016FieldawareFM, Rendle2010FactorizationM}, and deep learning-based models \cite{Cheng2016WideD, Wang2017DeepC, Guo2017DeepFMAF}. Moreover, several techniques specifically address unique challenges in CVR prediction, such as delayed feedback \cite{Chapelle2014ModelingDF, Su2020AnAM}, data sparsity \cite{Ma2018EntireSM, Wen2019EntireSM}, and selection bias \cite{guo2021enhanced, Zhang2019LargescaleCA}. This paper focuses primarily on mitigating selection bias issues.

\subsection{Recommendation with Selection Bias}

Bias in recommendation systems is a significant concern in current research \cite{chen2023bias, zhao2022popularity, wang2021samwalker++}, impacting the fairness and diversity of recommendations.

Selection bias, particularly missing-not-at-random, is common in recommender systems where feedback is observed only for displayed user-item pairs \cite{marlin2009collaborative, sato2020unbiased, yang2015boosting}. To mitigate this bias, the inverse propensity score (IPS) approach \cite{schnabel2016recommendations, swaminathan2015self} re-weights observed samples using inverse displayed probabilities. However, IPS estimators often suffer from high variance \cite{gilotte2018offline}, which can be mitigated by self-normalized inverse propensity score (SNIPS) estimators \cite{schnabel2016recommendations}. 

Note that the inherent nature of selection bias is that the data is missing not at random. A straightforward solution for selection bias is to impute the missing entries with pseudo-labels, aiming to make the observed data distribution $p(u, i| o=1)$ resemble the ideal uniform distribution $p(u, i)$. For instance, \cite{steck2010training, steck2013evaluation} propose a light imputation strategy that directly assigns a specific value to missing data. However, since these imputed ratings are heuristic, such methods often suffer from empirical inaccuracies, which can propagate into the training of recommendation models, resulting in sub-optimal performance.

Doubly Robust (DR) estimators \cite{jiang2016doubly, wang2019doubly} simultaneously account for imputation errors and propensities to reduce variance in IPS. Recent improvements include asymmetric tri-training \cite{saito2020asymmetric}, information theory considerations \cite{wang2020information}, adversarial training \cite{xu2020adversarial}, enhanced doubly robust estimators \cite{guo2021enhanced}, knowledge distillation \cite{xu2022ukd}, bias-variance trade-off \cite{dai2022generalized}, and multi-task learning \cite{wang2022escm}.  {DR-V2~\cite{pmlr-v202-li23ah} proposes balanced-mean-squared-error metric for joint propensity and CVR estimation.  \cite{zhu2020unbiased}presents a novel combinational joint learning framework that simultaneously learns unbiased user-item relevance and propensity estimation to improve the accuracy of implicit recommender systems.\cite{xu2022dually} leverages both user and item perspectives to estimate propensity scores, addressing biases in sequential recommendation systems. \cite{su2024ddpo}  introduces DDPO, a framework that mitigates sample selection bias in post-click conversion rate estimation by optimizing models with both clicked and unclicked samples in the impression space.}

Some approaches rely on small amounts of randomly unbiased data \cite{bonner2018causal, yuan2019improving, chen2021autodebias}, which can be costly in real-world applications.

Existing models often face challenges in calibrating propensity score estimations, leading to inaccuracies in debiasing methods like IPS and DR. Addressing these calibration issues is the primary focus of this paper.

% {In the context of the latest advancements in recommender systems of the self-supervised learning methods~\cite{wei2022contrastive,xia2022self}, our calibration framework offers a significant enhancement. We discuss how our model-agnostic calibration can be seamlessly integrated with these systems to improve prediction accuracy and uncertainty estimates, providing a comprehensive enhancement to the robustness and reliability of recommendations. This integration not only boosts performance but also extends the applicability of our approach across diverse recommender system architectures.}

\subsection{Uncertainty Calibration and Quantification}

Beyond Platt scaling, temperature scaling has been proposed for uncertainty calibration in multi-class classification \cite{guo2017calibration}. Platt scaling and temperature scaling both assume a Gaussian distribution. For distributions that are richer and more skewed, Beta calibration is another effective method \cite{kull2017beta}. 
In addition to parametric methods that assume distributional assumptions, non-parametric techniques can also be considered. These include histogram binning \cite{zadrozny2001obtaining} and isotonic regression \cite{zadrozny2002transforming}.

From a Bayesian generative model perspective, this paper focuses on approximate methods like MC Dropout and Deep Ensembles for uncertainty quantification in IPS. While Gaussian processes, Bayesian neural networks, and other probabilistic graphical models can also quantify uncertainty in IPS \cite{zhu2017big,zhu2015bayes}, practical approximate inference on large-scale datasets poses significant challenges \cite{liao2020uncertainty, mccandless2009bayesian}. 
Different methods can excel in specific scenarios, but each has its limitations. A summary of common drawbacks for various approaches is deferred to Section V in the Supplemental Materials.

\section{Conclusions}
 {
This paper introduces Expected Calibration Error (ECE) as a novel metric to evaluate the reliability of propensity scores, shedding light on the prevalent issue of uncertainty miscalibration in recommendation systems where data is missing not at random (MNAR). To address this challenge, we propose uncertainty calibration techniques for propensity score estimation and systematically compare three calibration approaches.}

 {
Through theoretical analysis, we demonstrate that calibrated Inverse Propensity Scores (IPS) reduce bias, leading to more reliable debiasing in recommendation tasks. Extensive experiments on three benchmark datasets, Coat Shopping, Yahoo! R3 and KuaiRand, validate the effectiveness of our methods, showing that calibrated propensity scores significantly enhance recommendation accuracy.
}

 {
These findings emphasize the critical role of addressing propensity score miscalibration in improving both bias mitigation and the overall performance of recommendation systems. This work highlights a promising direction for incorporating uncertainty calibration to ensure more robust and fair recommendations.
}

\bibliography{refs}

\begin{thebibliography}{10}

\bibitem{guo2021enhanced}
S.~Guo, L.~Zou, Y.~Liu, W.~Ye, S.~Cheng, S.~Wang, H.~Chen, D.~Yin, and
  Y.~Chang, ``Enhanced doubly robust learning for debiasing post-click
  conversion rate estimation,'' in {\em Proceedings of the 44th International
  ACM SIGIR Conference on Research and Development in Information Retrieval},
  pp.~275--284, 2021.

\bibitem{dai2022generalized}
Q.~Dai, H.~Li, P.~Wu, Z.~Dong, X.-H. Zhou, R.~Zhang, R.~Zhang, and J.~Sun, ``A
  generalized doubly robust learning framework for debiasing post-click
  conversion rate prediction,'' in {\em Proceedings of the 28th ACM SIGKDD
  Conference on Knowledge Discovery and Data Mining}, pp.~252--262, 2022.

\bibitem{wang2022escm}
H.~Wang, T.-W. Chang, T.~Liu, J.~Huang, Z.~Chen, C.~Yu, R.~Li, and W.~Chu,
  ``Escm $^{2}$: Entire space counterfactual multi-task model for post-click
  conversion rate estimation,'' in {\em SIGIR}, 2022.

\bibitem{zhou2023generalized}
Y.~Zhou, T.~Feng, M.~Liu, and Z.~Zhu, ``A generalized propensity learning
  framework for unbiased post-click conversion rate estimation,'' in {\em
  Proceedings of the 32nd ACM International Conference on Information and
  Knowledge Management}, pp.~3554--3563, 2023.

\bibitem{liu2022rmt}
Q.~Liu, Y.~Luo, S.~Wu, Z.~Zhang, X.~Yue, H.~Jin, and L.~Wang, ``Rmt-net:
  Reject-aware multi-task network for modeling missing-not-at-random data in
  financial credit scoring,'' {\em IEEE Transactions on Knowledge and Data
  Engineering}, 2022.

\bibitem{seaman2013review}
S.~R. Seaman and I.~R. White, ``Review of inverse probability weighting for
  dealing with missing data,'' {\em Statistical methods in medical research},
  vol.~22, no.~3, pp.~278--295, 2013.

\bibitem{little2019statistical}
R.~J. Little and D.~B. Rubin, {\em Statistical analysis with missing data},
  vol.~793.
\newblock John Wiley \& Sons, 2019.

\bibitem{swaminathan2015self}
A.~Swaminathan and T.~Joachims, ``The self-normalized estimator for
  counterfactual learning,'' {\em advances in neural information processing
  systems}, vol.~28, 2015.

\bibitem{schnabel2016recommendations}
T.~Schnabel, A.~Swaminathan, A.~Singh, N.~Chandak, and T.~Joachims,
  ``Recommendations as treatments: Debiasing learning and evaluation,'' in {\em
  international conference on machine learning}, pp.~1670--1679, PMLR, 2016.

\bibitem{wang2019doubly}
X.~Wang, R.~Zhang, Y.~Sun, and J.~Qi, ``Doubly robust joint learning for
  recommendation on data missing not at random,'' in {\em International
  Conference on Machine Learning}, pp.~6638--6647, PMLR, 2019.

\bibitem{saito2020doubly}
Y.~Saito, ``Doubly robust estimator for ranking metrics with post-click
  conversions,'' in {\em Fourteenth ACM Conference on Recommender Systems},
  pp.~92--100, 2020.

\bibitem{guo2017calibration}
C.~Guo, G.~Pleiss, Y.~Sun, and K.~Q. Weinberger, ``On calibration of modern
  neural networks,'' in {\em International conference on machine learning},
  pp.~1321--1330, PMLR, 2017.

\bibitem{vaicenavicius2019evaluating}
J.~Vaicenavicius, D.~Widmann, C.~Andersson, F.~Lindsten, J.~Roll, and
  T.~Sch{\"o}n, ``Evaluating model calibration in classification,'' in {\em The
  22nd International Conference on Artificial Intelligence and Statistics},
  pp.~3459--3467, PMLR, 2019.

\bibitem{menon2012predicting}
A.~K. Menon, X.~Jiang, S.~Vembu, C.~Elkan, and L.~Ohno-Machado, ``Predicting
  accurate probabilities with a ranking loss,'' in {\em Proceedings of the 29th
  International Coference on International Conference on Machine Learning},
  pp.~659--666, 2012.

\bibitem{kweon2022obtaining}
W.~Kweon, S.~Kang, and H.~Yu, ``Obtaining calibrated probabilities with
  personalized ranking models,'' in {\em Proceedings of the AAAI Conference on
  Artificial Intelligence}, vol.~36, pp.~4083--4091, 2022.

\bibitem{wei2022posterior}
P.~Wei, W.~Zhang, R.~Hou, J.~Liu, S.~Liu, L.~Wang, and B.~Zheng, ``Posterior
  probability matters: Doubly-adaptive calibration for neural predictions in
  online advertising,'' {\em arXiv preprint arXiv:2205.07295}, 2022.

\bibitem{xu2022ukd}
Z.~Xu, P.~Wei, W.~Zhang, S.~Liu, L.~Wang, and B.~Zheng, ``Ukd: Debiasing
  conversion rate estimation via uncertainty-regularized knowledge
  distillation,'' in {\em ACM Web Conference}, pp.~2078--2087, 2022.

\bibitem{marlin2007collaborative}
B.~Marlin, R.~S. Zemel, S.~Roweis, and M.~Slaney, ``Collaborative filtering and
  the missing at random assumption,'' in {\em Conference on Uncertainty in
  Artificial Intelligence}, p.~267–275, 2007.

\bibitem{sato2020unbiased}
M.~Sato, S.~Takemori, J.~Singh, and T.~Ohkuma, ``Unbiased learning for the
  causal effect of recommendation,'' in {\em Fourteenth ACM Conference on
  Recommender Systems}, pp.~378--387, 2020.

\bibitem{lecun2015deep}
Y.~LeCun, Y.~Bengio, and G.~Hinton, ``Deep learning,'' {\em nature}, vol.~521,
  no.~7553, pp.~436--444, 2015.

\bibitem{silver2016mastering}
D.~Silver, A.~Huang, C.~J. Maddison, A.~Guez, L.~Sifre, G.~Van Den~Driessche,
  J.~Schrittwieser, I.~Antonoglou, V.~Panneershelvam, M.~Lanctot, {\em et~al.},
  ``Mastering the game of go with deep neural networks and tree search,'' {\em
  nature}, vol.~529, no.~7587, pp.~484--489, 2016.

\bibitem{liu2017multi}
Q.~Liu, S.~Wu, and L.~Wang, ``Multi-behavioral sequential prediction with
  recurrent log-bilinear model,'' {\em IEEE Transactions on Knowledge and Data
  Engineering}, vol.~29, no.~6, pp.~1254--1267, 2017.

\bibitem{kendall2017uncertainties}
A.~Kendall and Y.~Gal, ``What uncertainties do we need in {B}ayesian deep
  learning for computer vision?,'' {\em Advances in neural information
  processing systems}, vol.~30, 2017.

\bibitem{abdar2021review}
M.~Abdar, F.~Pourpanah, S.~Hussain, D.~Rezazadegan, L.~Liu, M.~Ghavamzadeh,
  P.~Fieguth, X.~Cao, A.~Khosravi, U.~R. Acharya, {\em et~al.}, ``A review of
  uncertainty quantification in deep learning: Techniques, applications and
  challenges,'' {\em Information Fusion}, vol.~76, pp.~243--297, 2021.

\bibitem{leibig2017leveraging}
C.~Leibig, V.~Allken, M.~S. Ayhan, P.~Berens, and S.~Wahl, ``Leveraging
  uncertainty information from deep neural networks for disease detection,''
  {\em Scientific reports}, vol.~7, no.~1, pp.~1--14, 2017.

\bibitem{michelmore2020uncertainty}
R.~Michelmore, M.~Wicker, L.~Laurenti, L.~Cardelli, Y.~Gal, and M.~Kwiatkowska,
  ``Uncertainty quantification with statistical guarantees in end-to-end
  autonomous driving control,'' in {\em 2020 IEEE International Conference on
  Robotics and Automation (ICRA)}, pp.~7344--7350, IEEE, 2020.

\bibitem{gneiting2005weather}
T.~Gneiting and A.~E. Raftery, ``Weather forecasting with ensemble methods,''
  {\em Science}, vol.~310, no.~5746, pp.~248--249, 2005.

\bibitem{taylor2002neural}
J.~W. Taylor and R.~Buizza, ``Neural network load forecasting with weather
  ensemble predictions,'' {\em IEEE Transactions on Power systems}, vol.~17,
  no.~3, pp.~626--632, 2002.

\bibitem{akcora2019graphboot}
C.~G. Akcora, Y.~R. Gel, M.~Kantarcioglu, V.~Lyubchich, and B.~Thuraisingham,
  ``Graphboot: Quantifying uncertainty in node feature learning on large
  networks,'' {\em IEEE Transactions on Knowledge and Data Engineering},
  vol.~33, no.~1, pp.~116--127, 2019.

\bibitem{xu2024calibrated}
H.~Xu, Y.~Wang, S.~Jian, Q.~Liao, Y.~Wang, and G.~Pang, ``Calibrated one-class
  classification for unsupervised time series anomaly detection,'' {\em IEEE
  Transactions on Knowledge and Data Engineering}, 2024.

\bibitem{qian2023towards}
W.~Qian, Y.~Zhao, D.~Zhang, B.~Chen, K.~Zheng, and X.~Zhou, ``Towards a unified
  understanding of uncertainty quantification in traffic flow forecasting,''
  {\em IEEE Transactions on Knowledge and Data Engineering}, 2023.

\bibitem{kull2017beta}
M.~Kull, T.~Silva~Filho, and P.~Flach, ``Beta calibration: a well-founded and
  easily implemented improvement on logistic calibration for binary
  classifiers,'' in {\em Artificial Intelligence and Statistics}, pp.~623--631,
  PMLR, 2017.

\bibitem{wilson2020bayesian}
A.~G. Wilson and P.~Izmailov, ``Bayesian deep learning and a probabilistic
  perspective of generalization,'' {\em Advances in neural information
  processing systems}, vol.~33, pp.~4697--4708, 2020.

\bibitem{wang2020survey}
H.~Wang and D.-Y. Yeung, ``A survey on {B}ayesian deep learning,'' {\em ACM
  Computing Surveys (CSUR)}, vol.~53, no.~5, pp.~1--37, 2020.

\bibitem{lakshminarayanan2017simple}
B.~Lakshminarayanan, A.~Pritzel, and C.~Blundell, ``Simple and scalable
  predictive uncertainty estimation using deep ensembles,'' {\em Advances in
  neural information processing systems}, vol.~30, 2017.

\bibitem{platt1999probabilistic}
J.~Platt {\em et~al.}, ``Probabilistic outputs for support vector machines and
  comparisons to regularized likelihood methods,'' {\em Advances in large
  margin classifiers}, vol.~10, no.~3, pp.~61--74, 1999.

\bibitem{niculescu2005predicting}
A.~Niculescu-Mizil and R.~Caruana, ``Predicting good probabilities with
  supervised learning,'' in {\em Proceedings of the 22nd international
  conference on Machine learning}, pp.~625--632, 2005.

\bibitem{kuleshov2018accurate}
V.~Kuleshov, N.~Fenner, and S.~Ermon, ``Accurate uncertainties for deep
  learning using calibrated regression,'' in {\em International conference on
  machine learning}, pp.~2796--2804, PMLR, 2018.

\bibitem{cui2020calibrated}
P.~Cui, W.~Hu, and J.~Zhu, ``Calibrated reliable regression using maximum mean
  discrepancy,'' {\em Advances in Neural Information Processing Systems},
  vol.~33, pp.~17164--17175, 2020.

\bibitem{he2023investigating}
G.~He, P.~Cui, J.~Chen, W.~Hu, and J.~Zhu, ``Investigating uncertainty
  calibration of aligned language models under the multiple-choice setting,''
  {\em arXiv preprint arXiv:2310.11732}, 2023.

\bibitem{liu2023deep}
Y.~Liu, P.~Cui, W.~Hu, and R.~Hong, ``Deep ensembles meets quantile regression:
  Uncertainty-aware imputation for time series,'' {\em arXiv preprint
  arXiv:2312.01294}, 2023.

\bibitem{li2024conformalized}
P.~Li, L.~Hua, Z.~Ma, W.~Hu, Y.~Liu, and J.~Zhu, ``Conformalized graph learning
  for molecular admet property prediction and reliable uncertainty
  quantification,'' {\em Journal of Chemical Information and Modeling}, 2024.

\bibitem{cui2024sde}
P.~Cui, Z.~Deng, W.~Hu, and J.~Zhu, ``{SDE-HNN}: Accurate and well-calibrated
  forecasting using stochastic differential equations,'' {\em ACM Trans. Knowl.
  Discov. Data}, 2024.

\bibitem{chen2024unveiling}
Z.~Chen, W.~Hu, G.~He, Z.~Deng, Z.~Zhang, and R.~Hong, ``Unveiling uncertainty:
  A deep dive into calibration and performance of multimodal large language
  models,'' {\em arXiv preprint arXiv:2412.14660}, 2024.

\bibitem{vermeulen2015bias}
K.~Vermeulen and S.~Vansteelandt, ``Bias-reduced doubly robust estimation,''
  {\em Journal of the American Statistical Association}, vol.~110, no.~511,
  pp.~1024--1036, 2015.

\bibitem{gal2016dropout}
Y.~Gal and Z.~Ghahramani, ``Dropout as a {B}ayesian approximation: Representing
  model uncertainty in deep learning,'' in {\em international conference on
  machine learning}, pp.~1050--1059, PMLR, 2016.

\bibitem{tao2023dual}
L.~Tao, M.~Dong, and C.~Xu, ``Dual focal loss for calibration,'' in {\em
  International Conference on Machine Learning}, pp.~33833--33849, PMLR, 2023.

\bibitem{wen2019batchensemble}
Y.~Wen, D.~Tran, and J.~Ba, ``Batchensemble: an alternative approach to
  efficient ensemble and lifelong learning,'' in {\em International Conference
  on Learning Representations}, 2019.

\bibitem{bai2021don}
Y.~Bai, S.~Mei, H.~Wang, and C.~Xiong, ``Don’t just blame
  over-parametrization for over-confidence: Theoretical analysis of calibration
  in binary classification,'' in {\em International Conference on Machine
  Learning}, pp.~566--576, PMLR, 2021.

\bibitem{kumar2019verified}
A.~Kumar, P.~S. Liang, and T.~Ma, ``Verified uncertainty calibration,'' {\em
  Advances in Neural Information Processing Systems}, vol.~32, 2019.

\bibitem{song2023cdr}
Z.~Song, J.~Chen, S.~Zhou, Q.~Shi, Y.~Feng, C.~Chen, and C.~Wang, ``Cdr:
  Conservative doubly robust learning for debiased recommendation,'' in {\em
  Proceedings of the 32nd ACM International Conference on Information and
  Knowledge Management}, pp.~2321--2330, 2023.

\bibitem{he2017neural}
X.~He, L.~Liao, H.~Zhang, L.~Nie, X.~Hu, and T.-S. Chua, ``Neural collaborative
  filtering,'' in {\em Proceedings of the 26th international conference on
  world wide web}, pp.~173--182, 2017.

\bibitem{pmlr-v202-li23ah}
H.~Li, Y.~Xiao, C.~Zheng, P.~Wu, and P.~Cui, ``Propensity matters: Measuring
  and enhancing balancing for recommendation,'' in {\em Proceedings of the 40th
  International Conference on Machine Learning} (A.~Krause, E.~Brunskill,
  K.~Cho, B.~Engelhardt, S.~Sabato, and J.~Scarlett, eds.), vol.~202 of {\em
  Proceedings of Machine Learning Research}, pp.~20182--20194, PMLR, 23--29 Jul
  2023.

\bibitem{10.1145/3583780.3614760}
Y.~Zhou, T.~Feng, M.~Liu, and Z.~Zhu, ``A generalized propensity learning
  framework for unbiased post-click conversion rate estimation,'' in {\em
  Proceedings of the 32nd ACM International Conference on Information and
  Knowledge Management}, CIKM '23, (New York, NY, USA), p.~3554–3563,
  Association for Computing Machinery, 2023.

\bibitem{Richardson2007PredictingCE}
M.~Richardson, E.~Dominowska, and R.~J. Ragno, ``Predicting clicks: estimating
  the click-through rate for new ads,'' in {\em The Web Conference}, 2007.

\bibitem{Juan2016FieldawareFM}
Y.-C. Juan, Y.~Zhuang, W.-S. Chin, and C.-J. Lin, ``Field-aware factorization
  machines for ctr prediction,'' {\em Proceedings of the 10th ACM Conference on
  Recommender Systems}, 2016.

\bibitem{Rendle2010FactorizationM}
S.~Rendle, ``Factorization machines,'' {\em 2010 IEEE International Conference
  on Data Mining}, pp.~995--1000, 2010.

\bibitem{Cheng2016WideD}
H.-T. Cheng, L.~Koc, J.~Harmsen, T.~Shaked, T.~Chandra, H.~B. Aradhye,
  G.~Anderson, G.~S. Corrado, W.~Chai, M.~Ispir, R.~Anil, Z.~Haque, L.~Hong,
  V.~Jain, X.~Liu, and H.~Shah, ``Wide \& deep learning for recommender
  systems,'' {\em Proceedings of the 1st Workshop on Deep Learning for
  Recommender Systems}, 2016.

\bibitem{Wang2017DeepC}
R.~Wang, B.~Fu, G.~Fu, and M.~Wang, ``Deep \& cross network for ad click
  predictions,'' {\em Proceedings of the ADKDD'17}, 2017.

\bibitem{Guo2017DeepFMAF}
H.~Guo, R.~Tang, Y.~Ye, Z.~Li, and X.~He, ``Deepfm: A factorization-machine
  based neural network for ctr prediction,'' {\em ArXiv}, vol.~abs/1703.04247,
  2017.

\bibitem{Chapelle2014ModelingDF}
O.~Chapelle, ``Modeling delayed feedback in display advertising,'' {\em
  Proceedings of the 20th ACM SIGKDD international conference on Knowledge
  discovery and data mining}, 2014.

\bibitem{Su2020AnAM}
Y.~Su, L.~Zhang, Q.~Dai, B.~Zhang, J.~Yan, D.~Wang, Y.~Bao, S.~Xu, Y.~He, and
  W.~P. Yan, ``An attention-based model for conversion rate prediction with
  delayed feedback via post-click calibration,'' in {\em International Joint
  Conference on Artificial Intelligence}, 2020.

\bibitem{Ma2018EntireSM}
X.~Ma, L.~Zhao, G.~Huang, Z.~Wang, Z.~Hu, X.~Zhu, and K.~Gai, ``Entire space
  multi-task model: An effective approach for estimating post-click conversion
  rate,'' {\em The 41st International ACM SIGIR Conference on Research \&
  Development in Information Retrieval}, 2018.

\bibitem{Wen2019EntireSM}
H.~Wen, J.~Zhang, Y.~Wang, F.~Lv, W.~Bao, Q.~Lin, and K.~Yang, ``Entire space
  multi-task modeling via post-click behavior decomposition for conversion rate
  prediction,'' {\em Proceedings of the 43rd International ACM SIGIR Conference
  on Research and Development in Information Retrieval}, 2019.

\bibitem{Zhang2019LargescaleCA}
W.~Zhang, W.~Bao, X.-Y. Liu, K.~Yang, Q.~Lin, H.~Wen, and R.~Ramezani,
  ``Large-scale causal approaches to debiasing post-click conversion rate
  estimation with multi-task learning,'' {\em Proceedings of The Web Conference
  2020}, 2019.

\bibitem{chen2023bias}
J.~Chen, H.~Dong, X.~Wang, F.~Feng, M.~Wang, and X.~He, ``Bias and debias in
  recommender system: A survey and future directions,'' {\em ACM Transactions
  on Information Systems}, vol.~41, no.~3, pp.~1--39, 2023.

\bibitem{zhao2022popularity}
Z.~Zhao, J.~Chen, S.~Zhou, X.~He, X.~Cao, F.~Zhang, and W.~Wu, ``Popularity
  bias is not always evil: Disentangling benign and harmful bias for
  recommendation,'' {\em IEEE Transactions on Knowledge and Data Engineering},
  vol.~35, no.~10, pp.~9920--9931, 2022.

\bibitem{wang2021samwalker++}
C.~Wang, J.~Chen, S.~Zhou, Q.~Shi, Y.~Feng, and C.~Chen, ``Samwalker++:
  Recommendation with informative sampling strategy,'' {\em IEEE Transactions
  on Knowledge and Data Engineering}, 2021.

\bibitem{marlin2009collaborative}
B.~M. Marlin and R.~S. Zemel, ``Collaborative prediction and ranking with
  non-random missing data,'' in {\em Proceedings of the third ACM conference on
  Recommender systems}, pp.~5--12, 2009.

\bibitem{yang2015boosting}
H.~Yang, G.~Ling, Y.~Su, M.~R. Lyu, and I.~King, ``Boosting response aware
  model-based collaborative filtering,'' {\em IEEE Transactions on Knowledge
  and Data Engineering}, vol.~27, no.~8, pp.~2064--2077, 2015.

\bibitem{gilotte2018offline}
A.~Gilotte, C.~Calauz{\`e}nes, T.~Nedelec, A.~Abraham, and S.~Doll{\'e},
  ``Offline a/b testing for recommender systems,'' in {\em Proceedings of the
  Eleventh ACM International Conference on Web Search and Data Mining},
  pp.~198--206, 2018.

\bibitem{steck2010training}
H.~Steck, ``Training and testing of recommender systems on data missing not at
  random,'' in {\em Proceedings of the 16th ACM SIGKDD international conference
  on Knowledge discovery and data mining}, pp.~713--722, 2010.

\bibitem{steck2013evaluation}
H.~Steck, ``Evaluation of recommendations: rating-prediction and ranking,'' in
  {\em Proceedings of the 7th ACM conference on Recommender systems},
  pp.~213--220, 2013.

\bibitem{jiang2016doubly}
N.~Jiang and L.~Li, ``Doubly robust off-policy value evaluation for
  reinforcement learning,'' in {\em International Conference on Machine
  Learning}, pp.~652--661, PMLR, 2016.

\bibitem{saito2020asymmetric}
Y.~Saito, ``Asymmetric tri-training for debiasing missing-not-at-random
  explicit feedback,'' in {\em SIGIR}, pp.~309--318, 2020.

\bibitem{wang2020information}
Z.~Wang, X.~Chen, R.~Wen, S.-L. Huang, E.~Kuruoglu, and Y.~Zheng, ``Information
  theoretic counterfactual learning from missing-not-at-random feedback,'' {\em
  NeurIPS}, pp.~1854--1864, 2020.

\bibitem{xu2020adversarial}
D.~Xu, C.~Ruan, E.~Korpeoglu, S.~Kumar, and K.~Achan, ``Adversarial
  counterfactual learning and evaluation for recommender system,'' {\em
  NeurIPS}, 2020.

\bibitem{zhu2020unbiased}
Z.~Zhu, Y.~He, Y.~Zhang, and J.~Caverlee, ``Unbiased implicit recommendation
  and propensity estimation via combinational joint learning,'' in {\em
  Proceedings of the 14th ACM Conference on Recommender Systems}, pp.~551--556,
  2020.

\bibitem{xu2022dually}
C.~Xu, J.~Xu, X.~Chen, Z.~Dong, and J.-R. Wen, ``Dually enhanced propensity
  score estimation in sequential recommendation,'' in {\em Proceedings of the
  31st ACM International Conference on Information \& Knowledge Management},
  pp.~2260--2269, 2022.

\bibitem{su2024ddpo}
H.~Su, L.~Meng, L.~Zhu, K.~Lu, and J.~Li, ``Ddpo: Direct dual propensity
  optimization for post-click conversion rate estimation,'' in {\em Proceedings
  of the 47th International ACM SIGIR Conference on Research and Development in
  Information Retrieval}, pp.~1179--1188, 2024.

\bibitem{bonner2018causal}
S.~Bonner and F.~Vasile, ``Causal embeddings for recommendation,'' in {\em
  RecSys}, pp.~104--112, 2018.

\bibitem{yuan2019improving}
B.~Yuan, J.-Y. Hsia, M.-Y. Yang, H.~Zhu, C.-Y. Chang, Z.~Dong, and C.-J. Lin,
  ``Improving ad click prediction by considering non-displayed events,'' in
  {\em Proceedings of the 28th ACM International Conference on Information and
  Knowledge Management}, pp.~329--338, 2019.

\bibitem{chen2021autodebias}
J.~Chen, H.~Dong, Y.~Qiu, X.~He, X.~Xin, L.~Chen, G.~Lin, and K.~Yang,
  ``Autodebias: Learning to debias for recommendation,'' in {\em SIGIR}, 2021.

\bibitem{zadrozny2001obtaining}
B.~Zadrozny and C.~Elkan, ``Obtaining calibrated probability estimates from
  decision trees and naive {B}ayesian classifiers,'' in {\em Icml}, vol.~1,
  pp.~609--616, Citeseer, 2001.

\bibitem{zadrozny2002transforming}
B.~Zadrozny and C.~Elkan, ``Transforming classifier scores into accurate
  multiclass probability estimates,'' in {\em Proceedings of the eighth ACM
  SIGKDD international conference on Knowledge discovery and data mining},
  pp.~694--699, 2002.

\bibitem{zhu2017big}
J.~Zhu, J.~Chen, W.~Hu, and B.~Zhang, ``Big learning with {B}ayesian methods,''
  {\em National Science Review}, vol.~4, no.~4, pp.~627--651, 2017.

\bibitem{zhu2015bayes}
J.~Zhu and W.~Hu, ``Recent advances in {Bayesian} machine learning,'' {\em
  Journal of Computer Research and Development}, vol.~52, no.~1, pp.~16--26,
  2015.

\bibitem{liao2020uncertainty}
S.~X. Liao and C.~M. Zigler, ``Uncertainty in the design stage of two-stage
  {Bayesian} propensity score analysis,'' {\em Statistics in medicine},
  vol.~39, no.~17, pp.~2265--2290, 2020.

\bibitem{mccandless2009bayesian}
L.~C. McCandless, P.~Gustafson, and P.~C. Austin, ``Bayesian propensity score
  analysis for observational data,'' {\em Statistics in medicine}, vol.~28,
  no.~1, pp.~94--112, 2009.

\end{thebibliography}
\bibliographystyle{ieeetr}

\section{Biography Section}
\vspace{-35pt}
\begin{IEEEbiography}[{\includegraphics[width=1in,height=1.25in,clip,keepaspectratio]{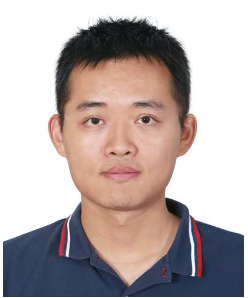}}]{Wenbo Hu}
 received his PhD from Tsinghua University in 2018. From 2018 to 2020, he worked as a postdoctoral researcher at Tsinghua University. He is currently an associate professor at Hefei University of Technology. He has published more than 20 outstanding conference and journal papers in his research areas, including multi-modal pre-training large models, AI against attack and defense, and AI uncertainty prediction. %
 %Dr. Hu received the 2017 Cross-Strait Tsinghua Academic Paper Award and the 2015 International East Asian Research Paper Computer Best Paper Award.
\end{IEEEbiography}

\vspace{-30pt}
\begin{IEEEbiography}[{\includegraphics[width=1in,height=1.25in,clip,keepaspectratio]{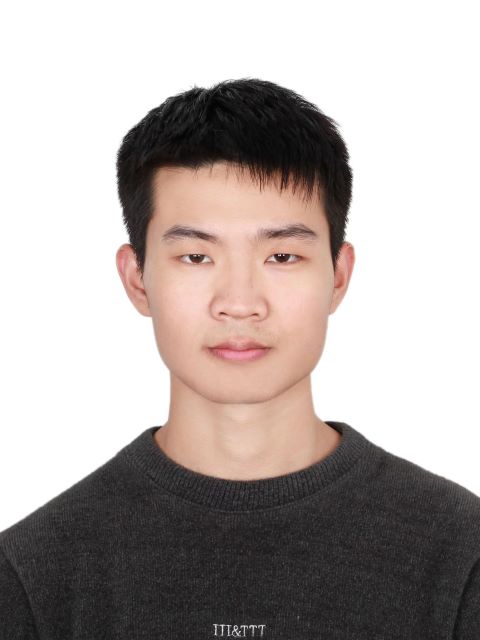}}]{Xin Sun}
is a joint Ph.D. candidate from University of Science and Technology of China(USTC) and Institute of Automation, Chinese Academy of Sciences(CASIA). He received his bachelor degree from Shanghai Jiao Tong University(SJTU). His current research interests mainly include trustworthy learning and information retrieval.
\end{IEEEbiography}

\vspace{-20pt}
\begin{IEEEbiography}[{\includegraphics[width=1in,height=1.25in,clip,keepaspectratio]{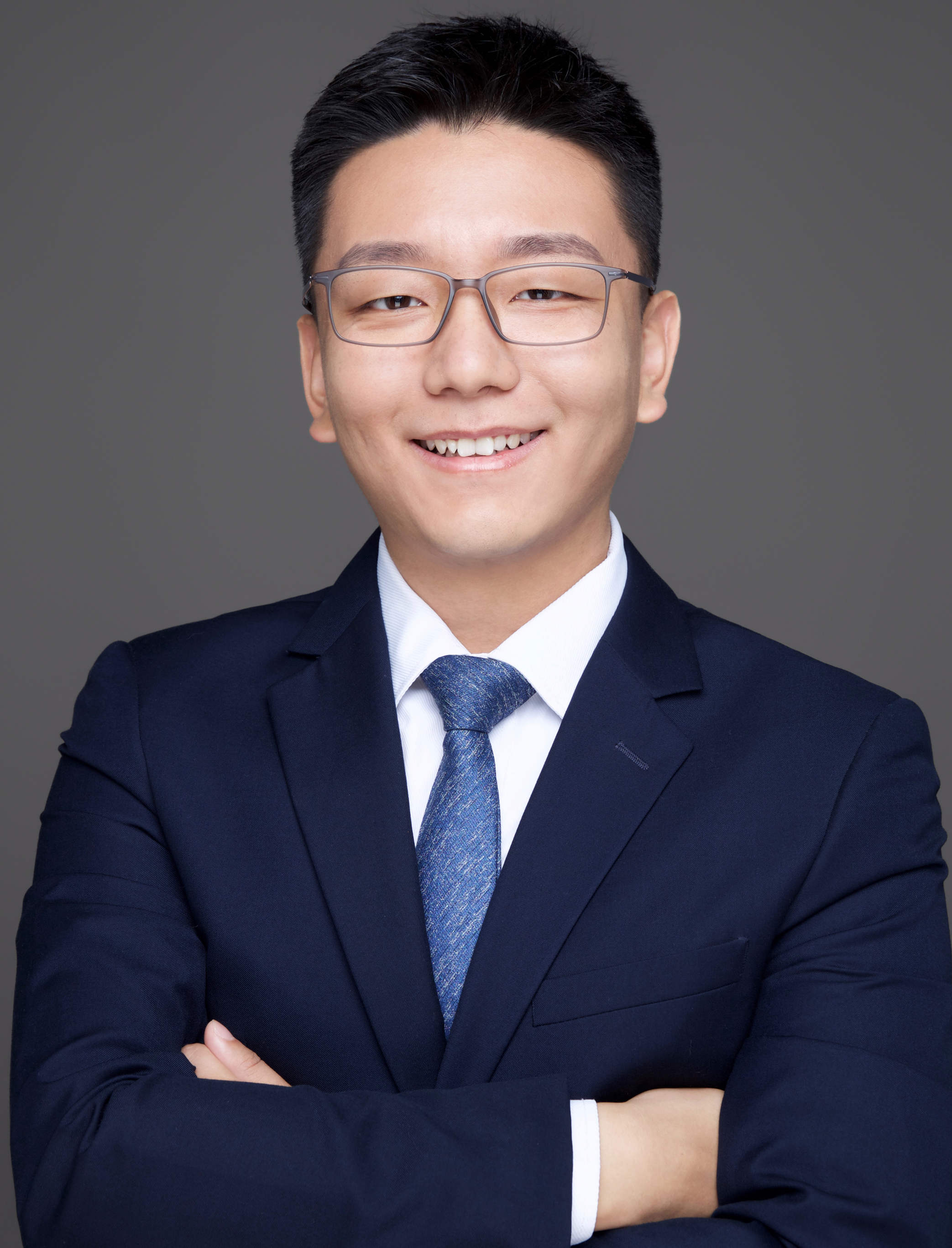}}]{Qiang Liu}
is an Associate Professor with the Center for Research on Intelligent Perception and Computing (CRIPAC), State Key Laboratory of Multimodal Artificial Intelligence Systems (MAIS), Institute of Automation, Chinese Academy of Sciences (CASIA). He received his
PhD degree from CASIA. Currently, his research interests include data mining, misinformation detection, LLM safety and AI for science. He has published papers in top-tier journals and conferences, such as IEEE TKDE, AAAI, NeurIPS, KDD, WWW, SIGIR, CIKM, ICDM, ACL and EMNLP.
\end{IEEEbiography}

\vspace{-20pt}
\begin{IEEEbiography}[{\includegraphics[width=1in,height=1.25in,clip,keepaspectratio]{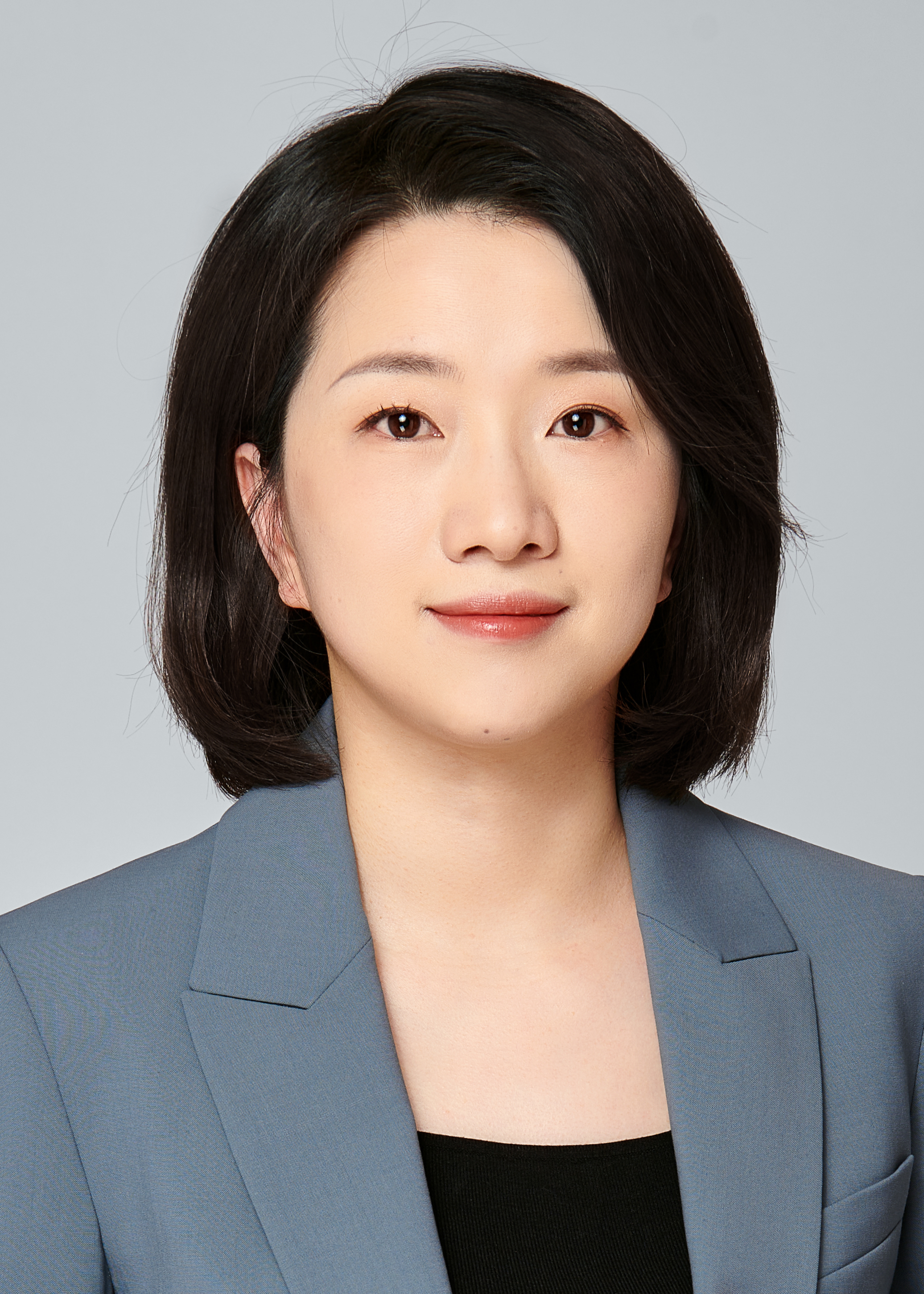}}]{Le Wu} is currently a professor at the Hefei
University of Technology (HFUT), China. She
received her Ph.D. degree from the University
of Science and Technology of China (USTC).
Her general area of research interests are data
mining, recommender systems, and social network analysis. She has published more than
50 papers in referred journals and conferences,
such as IEEE TKDE, SIGIR, WWW, and AAAI.
%Dr. Le Wu is the recipient of the Best of SDM
2015 Award, and the Distinguished Dissertation
Award from the China Association for Artificial Intelligence (CAAI) 2017.
\end{IEEEbiography}

\vspace{-20pt}
\begin{IEEEbiography}[{\includegraphics[width=1in,height=1.25in,clip,keepaspectratio]{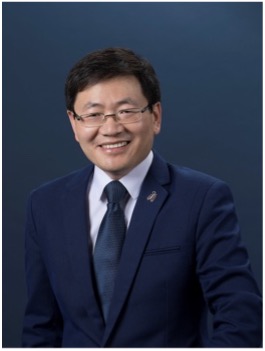}}]{Liang Wang}
received both the BEng and MEng degrees from Anhui University in 1997 and 2000, respectively, and the PhD degree from the Institute of Automation, Chinese Academy of Sciences (CASIA) in 2004. 
%From 2004 to 2010, he was a research assistant at Imperial College London, United Kingdom, and Monash University, Australia, a research fellow at the University of Melbourne, Australia, and a lecturer at the University of Bath, United Kingdom, respectively.
Currently, he is a full professor of the Hundred Talents Program at the State Key Laboratory of Multimodal Artificial Intelligence Systems, CASIA. 
His major research interests include machine learning, pattern recognition, and computer vision. He has widely published in highly ranked international journals such as IEEE TPAMI and IEEE TIP, and leading international conferences such as CVPR, ICCV, and ECCV. 
He has served as an Associate Editor of IEEE TPAMI, IEEE TIP, and PR. He is an IEEE Fellow and an IAPR Fellow.
\end{IEEEbiography}

\clearpage
\onecolumn
{\center \Large Supplemental Materials  for \\
{\Large Uncertainty Calibration for Counterfactual Propensity Estimation in Recommendation}}
\section{Summary of Main Symbols}
Table~\ref{tab:symbols} describes the main symbols used in this paper.
\label{appendix:symbols}
\begin{table}[htbp]
    \centering
    \begin{tabular}{cl}
    \toprule
    Symbol & Description \\
    \midrule
    $\mathcal{U}$  & Users set  \\
    $\mathcal{I}$  & Items set \\
    $\mathbf{R}$ &  conversion matrix \\
    $\mathbf{\hat{R}}$ & predicted conversion matrix \\
    $\mathcal{O}$ & click label matrix \\
    $\mathbf{R^{o}}$ & observed conversion matrix \\
    $\mathcal{D}$ & user-item pairs space \\
    $\mathbf{E}$ & prediction error matrix \\
    $\mathbf{\hat{E}}$ & imputed error matrix \\
    $\mathcal{P}$ & propensity scores matrix \\
    $\mathcal{\hat{P}}$ & estimated propensity scores matrix \\
    $g_\phi$ & propensity estimation model \\
    $f_\theta$ & CVR prediction model \\
    \bottomrule

    \end{tabular}
    \caption{The summary of main symbols used in this paper.}
    \label{tab:symbols}
\end{table}

\begin{figure*}[b]
    \centering
    \includegraphics[width=1.0\linewidth]{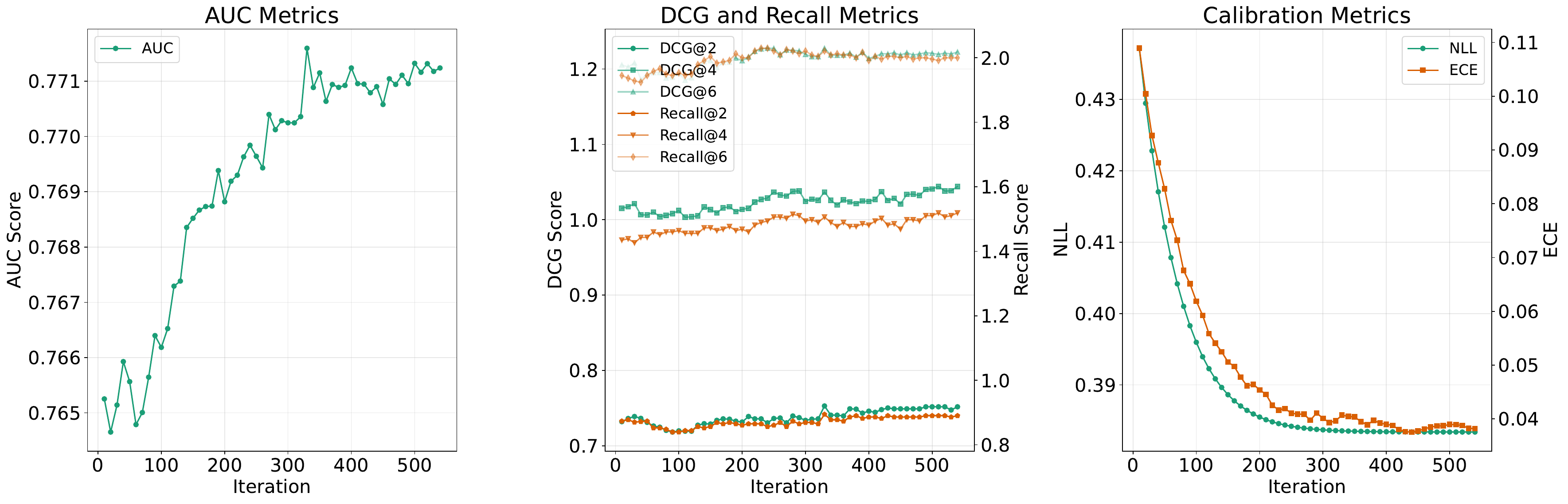}
    \caption{The relationship between ECE and recommendation metrics.}
    \label{fig:calib-metric}
\end{figure*}

\section{Model implementation}
\label{appendix:implementation}
We implement all models with Pytorch and optimize them with adam optimizer. We first determine the hyper-parameters for NeuMF(backbone) based on grid search, and the search range for the embedding size, batch size, learning rate, L2 regularization coefficient are set as {16, 32, 64, 128, 256}, {256, 512, 1024, 2048}, {5e-5, 1e-4, 5e-4, 1e-3, 5e-3, 1e-2} and {1e-5, 5e-5, 1e-4, 5e-4, 1e-3, 5e-3} respectively. The best configuration for each mothod is determined based on the ranking performance on the validation set. The search results is as follows: the embedding size, batch size, learning rate, dropout rate and  L2 regularization coefficient  are set to 64, 1024, 0.001, 0.2 and 1e-4 respectively. The structure of the MLP layers in NeuMF is set to [64, 32, 16]. Other models, including IPS, DR-JL, MRDR, CDR, GPL, ESCM$^2$ and DR-V2, are all built upon NeuMF. They use the same hyper-parameter settings as the baseline NeuMF for common hyper-parameters.

{For evaluating recommendation results, we employ three metrics: AUC,  discount cumulative gain (DCG) and Recall~\cite{guo2021enhanced}}.
DCG and Recall are defined as:
\begin{eqnarray}
    \label{eqn:dcg}
    {\rm DCG}(K) = \sum_{k=1}^K \frac{Rel_{k}}{\log_2 (k+1)}, \\
    \label{eqn:recall}
    {\rm Recall}(K) = \sum_{k=1}^K Rel_{k},
\end{eqnarray}
where $k$ represents the ranking order, $K$ is a hyperparameter of the DCG metric, and $Rel_k$ is a binary indicator indicating whether the $k$-th sample is a positive sample.

\section{The relationship between calibration error and recommendation results.}
{In this section, we have conducted a detailed analysis of how improved calibration impacts recommendation metrics. Specifically, we utilized Platt scaling as a post-hoc calibration method to adjust propensity scores. Platt scaling optimizes the negative log-likelihood (NLL) loss using the LGFBS optimizer, where a reduction in NLL is accompanied by a corresponding decrease in Expected Calibration Error (ECE).\\
To demonstrate the relationship between ECE and recommendation performance, we set checkpoints every 10 epochs during training, saving the propensity scores, ECE, and NLL values at each point. We then trained an Inverse Propensity Scoring (IPS) model on the Coat dataset using these saved propensity scores to evaluate recommendation metrics, including AUC, DCG, and Recall. The results, as illustrated in the figure~\ref{fig:calib-metric}, clearly show a strong negative correlation between ECE and recommendation performance: as ECE decreases, metrics such as AUC, DCG, and Recall exhibit significant improvement.\\
This empirical evidence supports our claim that better-calibrated propensity scores (lower ECE) lead to superior recommendation outcomes, thereby providing a clearer and more compelling link between calibration quality and CVR prediction enhancement. Hence, our findings affirm that reducing ECE improves IPS predictions and overall recommendation performance.
}

\section{Calibration Curve and Propensity Histogram of Calibrated Propensity scores on the Yahoo! R3 Shopping Dataset}
\begin{figure}[htbp!]
    \centering
    \subcaptionbox{Calibration Curve}{\includegraphics[width=0.23\textwidth]{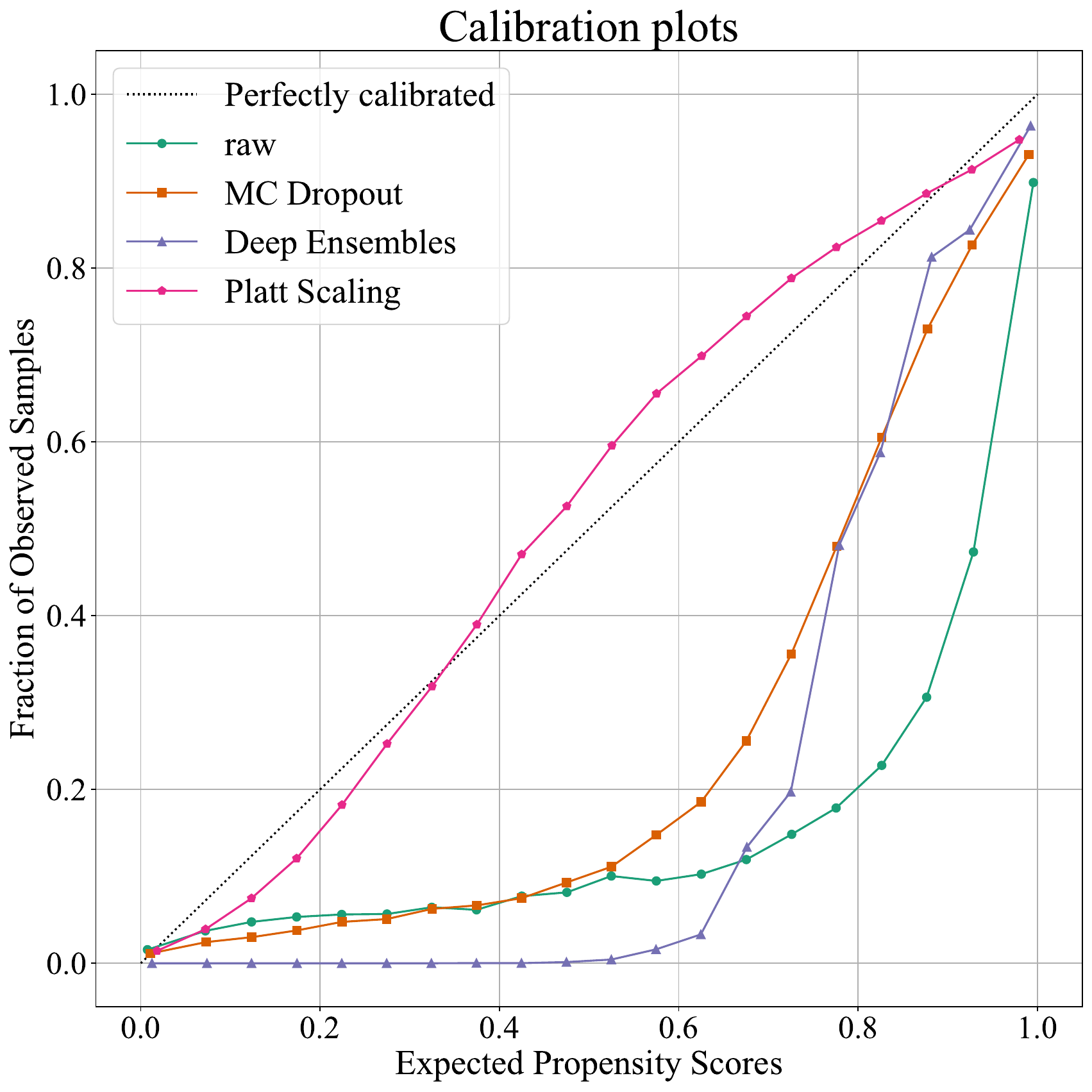}}
    \subcaptionbox{propensity scores Histogram}{\includegraphics[width=0.23\textwidth]{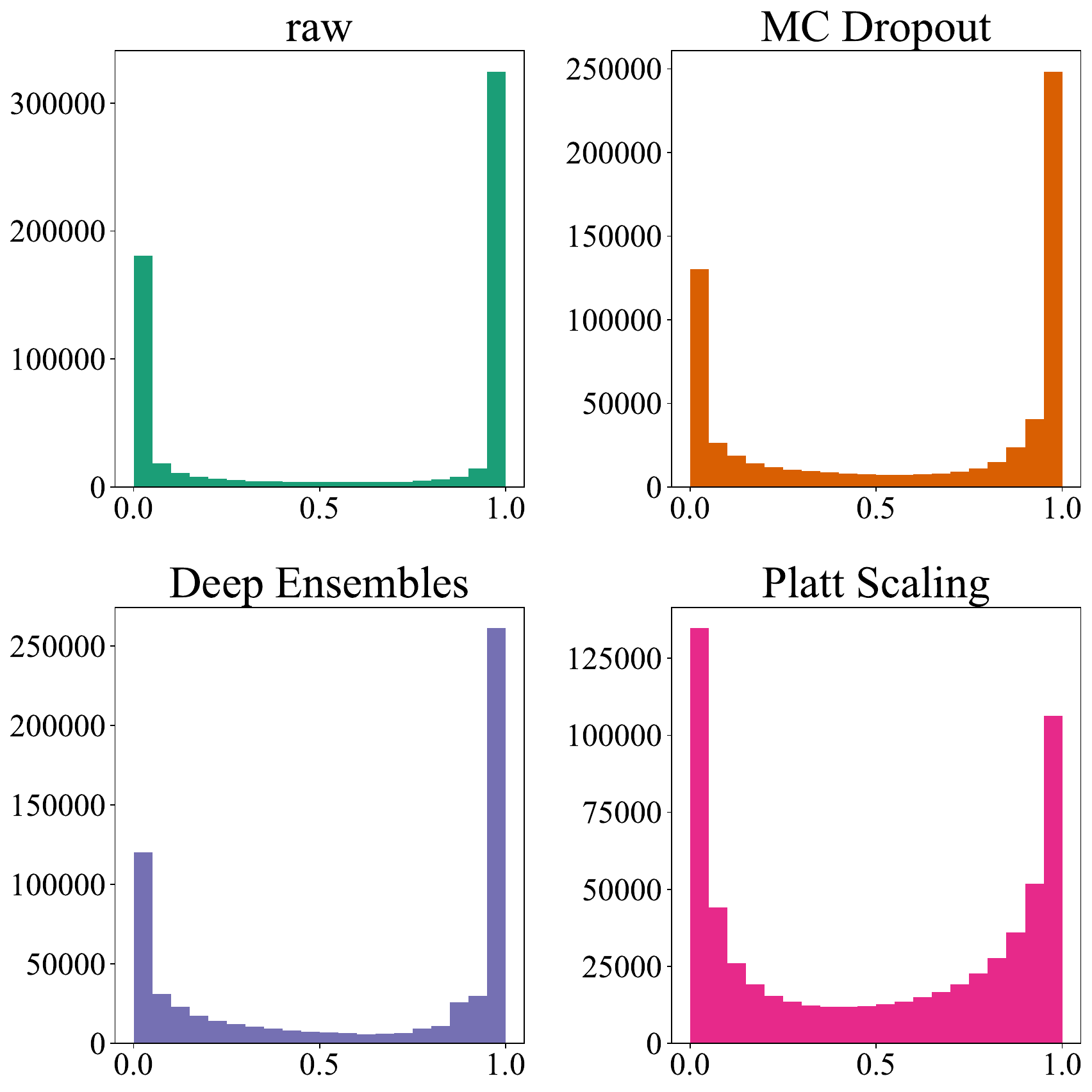}}
    \caption{Calibration Curve and Propensity Histogram of Calibrated Propensity scores on the Yahoo! R3 Shopping Dataset}
    \label{fig:yahoo-calibration}
\end{figure}

\section{Shortcomings of Each Calibration Method}
{Existing methods to address selection bias in recommender systems include deep ensembles, Platt scaling, and Monte Carlo Dropout. However, each comes with notable shortcomings:}
\begin{enumerate}
    \item {{Deep ensembles}, while providing robust uncertainty estimates, are \textit{computationally expensive} due to the necessity of training multiple models \cite{lakshminarayanan2017simple}. Hence the BatchEnsembles method can be used to reduce the overall computation cost as detailed in~\cite{wen2019batchensemble}.}
    \item {{Platt scaling} requires a \textit{separate validation set} to fine-tune its parameters, which can be a limitation in scenarios with limited data availability \cite{platt1999probabilistic}.}
    \item {{Monte Carlo Dropout} offers a practical approach to approximate Bayesian inference but can lead to \textit{unstable calibration performance}, particularly sensitive to the choice of dropout rate and the architecture of the underlying neural network \cite{gal2016dropout}.}
    \item Dual Focal Loss effectively addresses class imbalance and hard-to-classify examples but has notable shortcomings. Dual Focal Loss can overfit noisy data by overly focusing on mislabeled or ambiguous examples. It requires tuning of hyperparameters like the focusing factor, adding complexity to training~\cite{tao2023dual}.
\end{enumerate}

\iffalse
If you have an EPS/PDF photo (graphicx package needed), extra braces are
 needed around the contents of the optional argument to biography to prevent
 the LaTeX parser from getting confused when it sees the complicated
 $\backslash${\tt{includegraphics}} command within an optional argument. (You can create
 your own custom macro containing the $\backslash${\tt{includegraphics}} command to make things
 simpler here.)

\bf{If you include a photo:}
\begin{IEEEbiography}[{\includegraphics[width=1in,height=1.25in,clip,keepaspectratio]{fig1}}]{Michael Shell}
Use $\backslash${\tt{begin\{IEEEbiography\}}} and then for the 1st argument use $\backslash${\tt{includegraphics}} to declare and link the author photo.
Use the author name as the 3rd argument followed by the biography text.
\end{IEEEbiography}

\vspace{11pt}

\bf{If you will not include a photo:}
\begin{IEEEbiographynophoto}{John Doe}
Use $\backslash${\tt{begin\{IEEEbiographynophoto\}}} and the author name as the argument followed by the biography text.
\end{IEEEbiographynophoto}

\fi
\vfill
\end{document}